\documentclass[11pt,letterpaper]{article}
\usepackage[T1]{fontenc}
\usepackage{lmodern}
\usepackage[margin=1in]{geometry}
\usepackage{amsmath,amssymb,amsthm}
\usepackage{booktabs}
\usepackage{multirow}
\usepackage{array}
\usepackage{enumitem}
\usepackage{microtype}
\usepackage[hidelinks]{hyperref}
\hypersetup{
  pdftitle={Maximum Strong Independent Sets in Hypergraphs: Reductions, Bounds, and Greedy Certificates},
  pdfauthor={Yingquan (Cody) Wu and Jason Cong}
}

\theoremstyle{plain}
\newtheorem{theorem}{Theorem}
\newtheorem{lemma}{Lemma}
\newtheorem{proposition}{Proposition}
\newtheorem{corollary}{Corollary}
\theoremstyle{definition}
\newtheorem{definition}{Definition}
\newtheorem{algorithm}{Algorithm}
\newtheorem{example}{Example}
\newcounter{runningexampleroot}
\newcounter{runningexamplepart}
\renewcommand{\therunningexamplepart}{\arabic{runningexampleroot}(\alph{runningexamplepart})}

\newcommand{\startrunningexampleseries}{%
  \refstepcounter{example}%
  \setcounter{runningexampleroot}{\value{example}}%
  \setcounter{runningexamplepart}{0}%
}
\newenvironment{runningexamplepart}[1][]{%
  \refstepcounter{runningexamplepart}%
  \par\noindent\textbf{Example~\therunningexamplepart%
  \if\relax\detokenize{#1}\relax\else\ (#1)\fi.}\ }{\par}
\theoremstyle{remark}
\newtheorem{remark}{Remark}

\begin{document}

\title{Maximum Strong Independent Sets in Hypergraphs: Reductions, Bounds, and Greedy Certificates}
\author{
  Yingquan (Cody) Wu\thanks{Corresponding author.}\\
  {\small MBZUAI Institute of Foundation Models}\\
  {\small Sunnyvale, CA, USA}\\
  {\small\href{mailto:icodywu@gmail.com}{icodywu@gmail.com}}
  \and
  Jason Cong\\
  {\small University of California, Los Angeles}\\
  {\small Los Angeles, CA, USA}\\
  {\small\href{mailto:cong@cs.ucla.edu}{cong@cs.ucla.edu}}
}
\date{\small Preprint. Under review at the Journal of the ACM.}
\maketitle

\begin{abstract}
We study the maximum strong independent set problem in a finite hypergraph: find the largest vertex set that intersects every hyperedge in at most one vertex. This objective arises whenever each observed block is a local incompatibility constraint but transitive closure across overlapping blocks is not justified. A motivating example is multi-band LSH-MinHash deduplication, where each collision bucket gives local evidence, while connected-component contraction can impose spurious global equivalences. The paper develops an incidence-structural toolkit for this problem. We prove exact reductions for dominance, incidence twins, and weight-1 blocks; derive closed-form and low-weight upper bounds; introduce puncturing and covering certificates that sharpen those bounds; and analyze a layered greedy clustering algorithm driven by block weights and residual incidence. The algorithmic analysis includes feasibility, maximality, conditional optimality, a layered witness-matching upper bound, and incidence-local complexity bounds. The results give correctness, termination, fixed-point, and optimality certificates for broad incidence families, together with examples showing when different certificates separate or coincide.
\end{abstract}


\noindent\textbf{ACM CCS:} Theory of computation---Approximation algorithms analysis; Problems, reductions and completeness. Mathematics of computing---Hypergraphs.

\section{Introduction}

A strong independent set in a hypergraph is a vertex set that contains at most one vertex from each hyperedge. The corresponding maximum strong independent set problem is a natural abstraction for selection under local block constraints: a block asserts that its vertices are mutually incompatible for simultaneous selection, but overlaps between different blocks do not justify transitive contraction. This distinction is central in applications where blocks are generated by noisy or local evidence rather than by an equivalence relation.

One motivating example is large-scale LSH-MinHash deduplication \cite{broder1997,broder2000,charikar2002,manku2007}. A multi-band LSH pipeline emits collision buckets band by band. Each bucket is a local incompatibility block, but connected-component contraction over overlapping buckets can over-merge unrelated vertices through chains of local collisions. The correct combinatorial object is therefore the hypergraph whose vertices are items and whose hyperedges are emitted buckets; feasible retention is exactly strong independence. A companion application study reports ClueWeb22 (1.62B documents) and HPLT (3.98B documents) ablations where, across all full-corpus rows, the puncturing certificate is strictly stronger than the converged covering certificate, and the layered greedy solver reaches at least $99.43\%$ of the puncturing bound (within $0.5\%$ in 16 of 20 rows, and within $0.6\%$ in all rows) \cite{wu2026bucketfeasible}. The results below use this example only as motivation. The theory is stated for an arbitrary finite hypergraph $H=(V,\mathcal{B})$.

The same incidence-block objective also appears in fixed-path routing, weighted set packing and combinatorial-auction winner determination, and perfect minimal hashing \cite{alon1999routing,gottlob2013auctions,grossmann2026}. In all cases, feasibility is governed by materialized multiway conflicts rather than by transitive equivalence classes.

This paper develops a structural and algorithmic framework for maximum strong independent set in such incidence hypergraphs. The central new contribution is an incidence-certificate layer for bounding the optimum: a closed-form block-weight upper bound, its weight-1 refinement, layered support-superset sharpening, iterative puncturing certificates, and covering/reweighting certificates. These bounds are explicit functions of the materialized vertex-block incidence lists; they do not rely on solving a linear program (LP), semidefinite program (SDP), theta-body relaxation, or exact search subproblem. Around this certificate layer, we prove exact reductions based on dominance, incidence-twin compression, canonical handling of weight-1 blocks, and pivot dominance closure. We then give algorithms that actively improve and track the certificates: iterative greedy puncturing searches residual block subfamilies for stronger bounds, greedy covering constructs independent upper-bound witnesses, and greedy layered clustering uses the same block-weight order to build feasible root sets. The proofs give validity, termination, fixed-point, feasibility, maximality, conditional optimality, layered witness-matching, and incidence-local complexity guarantees.

The guiding principle is that the useful structure is not generic graph adjacency alone but the full vertex-block incidence relation. In particular, low block weight and incidence support determine both the strongest upper-bound certificates and the order in which the bound-improving and clustering algorithms should act. This perspective separates the structural objective from any particular source of blocks and makes the framework applicable to any setting in which local incompatibility constraints are materialized as a hyperedge family.

Maximum strong independent set inherits classical hardness phenomena from Karp's reductions \cite{karp1972}, and sits within a broad approximation, local-search, and relaxation literature for hypergraph packing problems \cite{hurkens1989,hazan2006,sviridenko2013,cygan2013,castrosilva2023}. Bounded-degree and uniform hypergraph independent-set algorithms give important general baselines \cite{hofmeister1998,halldorsson2009,agnarsson2013}, while recent work studies data reductions for strong maximum independent set in hypergraphs \cite{grossmann2026}. Our focus is complementary: the objective is classical, but the block-weight bounds, residual certificate-search algorithms, and certificate-guided clustering schedule are new incidence-structural tools.

\begin{table}[t]
\centering
\caption{Comparison with existing approaches for hypergraph strong independent set and set-packing-type formulations.}
\small
\setlength{\tabcolsep}{3pt}
\begin{tabular}{@{}>{\raggedright\arraybackslash}p{0.14\textwidth}>{\raggedright\arraybackslash}p{0.24\textwidth}>{\raggedright\arraybackslash}p{0.24\textwidth}>{\raggedright\arraybackslash}p{0.28\textwidth}@{}}
\toprule
Approach & Main object & Typical guarantee & Relation to this work \\
\midrule
Classical set packing and local search \cite{hurkens1989,hazan2006,sviridenko2013,cygan2013} &
Pack disjoint sets or improve a feasible packing through local exchanges. &
Approximation and hardness results, often parameterized by set size or exchange neighborhood. &
Defines a neighboring optimization landscape, but does not provide the block-weight incidence certificates used here. \\
\midrule
Bounded-degree or uniform hypergraph MIS \cite{hofmeister1998,halldorsson2009,agnarsson2013} &
Find large weak or strong independent sets under degree, rank, or uniformity assumptions. &
Approximation algorithms and limitations for generic hypergraph families. &
Provides generic baselines; our algorithms instead use materialized block incidences, minimum block weights, and layer-local updates. \\
\midrule
SDP/theta-style relaxations \cite{agnarsson2013,castrosilva2023} &
Relax independence constraints into semidefinite or theta-body programs. &
Optimization-based upper bounds or approximation tools. &
Complementary to our explicit closed-form and residual certificates, which require no relaxation solver. \\
\midrule
Exact preprocessing and data reduction \cite{grossmann2026} &
Reduce strong-MIS instances while preserving exact solvability. &
Kernelization/data-reduction effectiveness before exact or branch-based search. &
Related in spirit to dominance and incidence-twin compression, but our reductions are tied to computable upper-bound certificates and greedy clustering. \\
\midrule
Incidence-certificate framework (this work) &
Use block weights, incidence supports, dominance, twins, puncturing, and covering to bound and cluster. &
Closed-form, weight-1, puncturing, covering, witness-matching, and incidence-local complexity certificates. &
Targets materialized block families directly and explains when a block-weight greedy clustering schedule tracks the strongest certificates. \\
\bottomrule
\end{tabular}
\label{tab:prior-art-compare}
\end{table}

Table~\ref{tab:prior-art-compare} summarizes the distinction. Existing set-packing and hypergraph-MIS machinery typically optimizes a feasible packing by local exchanges, proves guarantees under rank/degree restrictions, relaxes the independence constraints into an auxiliary convex program, or reduces the instance before exact search. The contribution here is not another instantiation of those templates. We introduce an incidence-certificate layer whose quantities are computed directly from the given vertex--block incidence lists: the closed-form block-weight sum, canonical weight-1 preclusion, layered support-superset sharpening, iterative puncturing, covering reweighting, and greedy witness-matching certificates. These certificates upper-bound the optimum, explain when the bound is tight, and drive the block-native greedy clustering schedule without constructing a pairwise conflict graph, solving an LP/SDP/theta relaxation, or invoking a generic local-search oracle.

\section{Problem Formulation and Preliminaries}

Let $\mathcal{B}=\{B_1,\dots,B_M\}$ be a finite block family, and let
\[
V:=\bigcup_{B\in\mathcal{B}}B
\]
be the active universe. Any vertex that never appears in any block is unconstrained by $\mathcal{B}$ and is omitted from this optimization. Define the hypergraph
\begin{equation}
H=(V,\mathcal{B}),
\end{equation}
where elements of $V$ are vertices and blocks are hyperedges.
Throughout the theoretical sections, we use ``block'' and ``hyperedge'' interchangeably.
Accordingly, when we refer to a maximum independent set objective below, we mean the strong hypergraph version: no hyperedge may contain two selected vertices.

We use the standing convention that the block family is inclusion-reduced:
for distinct indices $i\neq j$,
\begin{equation}
B_i\not\subseteq B_j.
\end{equation}

Equivalently, if one forms the pairwise graph induced by block membership, each block $B$ is treated as a clique on the vertex set $B$: every pair in $B$ is an incompatibility pair under the strong-independence objective. This is a local selection constraint, not a transitive equivalence relation outside the stated block family.

\begin{definition}[Strong independent set and maximum strong independence number]
A set $R\subseteq V$ is a strong independent set of $H=(V,\mathcal{B})$ if
\begin{equation}
|R\cap B| \le 1 \qquad \forall B\in\mathcal{B}.
\end{equation}
The maximum strong independence number is
\begin{equation}
\alpha(\mathcal{B}) := \max\{ |R| : R\subseteq V,\ |R\cap B|\le 1\ \forall B\in\mathcal{B}\}.
\end{equation}
\end{definition}

To characterize the upper bounds and to derive our block-native greedy clustering algorithm, it is useful to record the incidence profile of each vertex. Intuitively, a vertex that appears in many blocks is globally more entangled: retaining it consumes feasibility across many local clique constraints at once.

\begin{definition}[block-index incidence set and vertex degree]
\label{def:incidence-degree}
For each vertex $v\in V$, define its block-index incidence set by
\begin{equation}
I(v):=\{i\in\{1,\dots,M\}: v\in B_i\},
\end{equation}
and its vertex degree by
\begin{equation}
d(v):=|I(v)|=|\{B\in\mathcal{B}:v\in B\}|.
\end{equation}
\end{definition}

\startrunningexampleseries
\begin{runningexamplepart}[Running example: block family and vertex degrees]
\label{ex:running-example-family}
We use the following block family as a running example in Sections~3--6:
\[
\begin{aligned}
B_1&=\{a,u\},\\
B_2&=\{b,u,f,p\},\qquad B_3=\{b,f,q\},\\
B_4&=\{e,g,p\},\qquad B_5=\{e,g,q\},\\
B_6&=\{c,x,p\},\qquad B_7=\{c,y,q\},\qquad B_8=\{c,x,y\},\\
B_9&=\{r,f,x,p\},\qquad B_{10}=\{r,g,y,q\},\\
B_{11}&=\{r,x,q\},\qquad B_{12}=\{r,y,p\}.
\end{aligned}
\]
Its block-overlap graph, where blocks are adjacent when they intersect, is connected, for example through
\[
B_1\cap B_2\neq\varnothing,\quad
B_2\cap B_4\neq\varnothing,\quad
B_4\cap B_6\neq\varnothing,\quad
B_6\cap B_9\neq\varnothing,\quad
B_9\cap B_{10}\neq\varnothing,\quad
B_{10}\cap B_{11}\neq\varnothing,
\]
so connected-component contraction would merge the whole instance into one connected component and retain only one representative. The vertex degrees are
\[
\begin{aligned}
d(a)&=1,\qquad d(b)=d(e)=d(u)=2,\\
d(c)&=d(f)=d(g)=3,\\
d(r)&=d(x)=d(y)=4,\qquad d(p)=d(q)=5.
\end{aligned}
\]
The vertices $u,f,g$ are deliberate support-superset distractors:
\[
\begin{aligned}
I(a)&=\{1\}\subsetneq I(u)=\{1,2\},\\
I(b)&=\{2,3\}\subsetneq I(f)=\{2,3,9\},\\
I(e)&=\{4,5\}\subsetneq I(g)=\{4,5,10\}.
\end{aligned}
\]
Later dominance and low-weight clustering rules will remove these distractors, revealing the block structure that drives the sharper certificate and the greedy clustering schedule.
\end{runningexamplepart}

\section{Characterizations of Maximum Strong Independent Sets}

This section records incidence reductions and canonical optimality conditions for maximum strong independent sets. The reductions either preserve $\alpha(\mathcal{B})$ exactly or identify vertices that can be fixed or suppressed without excluding all optima. These structural rules are the basis for the upper-bound certificates and greedy schedules in later sections.

A recurring structural statistic is the minimum incidence degree inside a block.

\begin{definition}[block weight]
For each block $B\in\mathcal{B}$, define
\begin{equation}
w(B) := \min_{v\in B} d(v).
\end{equation}
\end{definition}

\begin{runningexamplepart}[Running example: block weights]
\label{ex:running-example-weights}
For the running example in Example~\ref{ex:running-example-family}, the block weights are
\[
\begin{aligned}
w(B_1)&=1,\qquad
w(B_2)=\cdots=w(B_5)=2,\\
w(B_6)&=\cdots=w(B_{10})=3,\qquad
w(B_{11})=w(B_{12})=4.
\end{aligned}
\]
The apparent weight-$3$ status of $B_9$ and $B_{10}$ is caused only by the degree-$3$ support-superset distractors $f$ and $g$. The dominance rule below will justify deleting the distractors $u,f,g$ using $a,b,e$ as witnesses, respectively; in that reduced view the weights become
\[
\begin{aligned}
w_-(B_1)&=1,\qquad
w_-(B_2)=\cdots=w_-(B_5)=2,\\
w_-(B_6)&=w_-(B_7)=w_-(B_8)=3,\qquad
w_-(B_9)=\cdots=w_-(B_{12})=4.
\end{aligned}
\]
Thus the dominance-reduced family naturally decomposes into the intended weight layers $1,2,3,4$, while the unreduced family disguises two true weight-$4$ blocks as weight-$3$ blocks.
\end{runningexamplepart}

\begin{definition}[Incidence-twin quotient]
\label{def:incidence-twin-quotient}
Define an equivalence relation on the active vertex universe by
\begin{equation}
u\sim v
\quad\Longleftrightarrow\quad
I(u)=I(v).
\end{equation}
Since $V=\bigcup_{B\in\mathcal{B}}B$, every equivalence class has nonempty support. Let
\[
\overline V:=V/{\sim}
\]
be the set of incidence-twin classes, and write $[v]$ for the class containing $v$. For each block $B_i$, define its quotient block by
\begin{equation}
\overline B_i:=\{[v]:v\in B_i\}.
\end{equation}
The incidence-twin quotient family is
\begin{equation}
\overline{\mathcal{B}}:=\{\overline B_i: i=1,\dots,M\},
\end{equation}
viewed as an indexed block family, with duplicate quotient blocks removable by the inclusion rule above if desired.
\end{definition}

For the next theorem, write $\mathcal{B}=\{B_1,\dots,B_M\}$ and use the incidence sets $I(\cdot)$ from Definition~\ref{def:incidence-degree}.
We also use a deletion shorthand: for a vertex subset $S\subseteq V$, write
\[
\mathcal{B}-S
:=
\{B\setminus S:\ B\in\mathcal{B},\ B\setminus S\neq\varnothing\},
\]
with empty blocks discarded; for a single vertex $b$, write $\mathcal{B}-\{b\}$.
In the running example, the containments
\[
I(a)\subsetneq I(u),\qquad
I(b)\subsetneq I(f),\qquad
I(e)\subsetneq I(g)
\]
are the concrete cases to keep in mind: $u,f,g$ are more constrained support-superset vertices, so an optimum never needs to keep them instead of their smaller-support witnesses.

\begin{theorem}[Dominance and incidence-twin compression]
\label{thm:doc-dominance}
\label{thm:incidence-twin-compression}
The following two vertex-level reductions preserve the optimal retention value.
\begin{enumerate}[leftmargin=1.6em,label=(\roman*),itemsep=0.15em]
    \item \textbf{(Dominance removal)}
    If two vertices $a,b\in V$ satisfy
    \begin{equation}
    I(a)\subseteq I(b),
    \end{equation}
    then removing $b$ does not change the optimal retention value:
    \begin{equation}
    \alpha(\mathcal{B})=\alpha(\mathcal{B}-\{b\}).
    \end{equation}
    \item \textbf{(Incidence-twin compression)}
    For the quotient family in Definition~\ref{def:incidence-twin-quotient},
    \begin{equation}
    \alpha(\overline{\mathcal{B}})=\alpha(\mathcal{B}).
    \end{equation}
\end{enumerate}
\end{theorem}

\begin{proof}
For (i), any feasible retained set in $\mathcal{B}-\{b\}$ is feasible in $\mathcal{B}$, so
\(
\alpha(\mathcal{B}-\{b\})\le \alpha(\mathcal{B}).
\)
Let $O^\star$ be an optimal feasible set for $\mathcal{B}$.
If $b\notin O^\star$, then $O^\star$ is feasible for $\mathcal{B}-\{b\}$ and we are done.
Assume $b\in O^\star$ and define
\begin{equation}
O':=(O^\star\setminus\{b\})\cup\{a\}.
\end{equation}
For any index $i\in I(a)$, we have $i\in I(b)$ by $I(a)\subseteq I(b)$. Since $b\in O^\star$ and $O^\star$ is feasible, no other element of $O^\star$ can lie in $B_i$; replacing $b$ by $a$ keeps occupancy at most one in such blocks. All other blocks are unaffected. Hence $O'$ is feasible for $\mathcal{B}$ and $|O'|=|O^\star|$, with $b\notin O'$. Therefore there exists an optimal solution not using $b$, implying
\(
\alpha(\mathcal{B}-\{b\})\ge \alpha(\mathcal{B}).
\)
Combining both inequalities gives equality.

For (ii), let $\pi:V\to\overline V$ be the projection $\pi(v)=[v]$.
First let $R\subseteq V$ be feasible for $\mathcal{B}$. No two distinct vertices in $R$ can lie in the same incidence-twin class: if $u\sim v$ and $u\neq v$, then $I(u)=I(v)\neq\varnothing$, so $u$ and $v$ co-occur in every block indexed by this common support, contradicting feasibility. Hence $|\pi(R)|=|R|$. Moreover, if $\pi(R)$ contained two classes in some quotient block $\overline B_i$, then $R$ would contain two vertices in the original block $B_i$, again contradicting feasibility. Thus $\pi(R)$ is feasible for $\overline{\mathcal{B}}$, so $\alpha(\overline{\mathcal{B}})\ge\alpha(\mathcal{B})$.

Conversely, let $\overline R\subseteq\overline V$ be feasible for $\overline{\mathcal{B}}$, and choose one representative vertex $\rho(C)\in C$ from each class $C\in\overline R$. Set
\[
R:=\{\rho(C):C\in\overline R\}.
\]
If $R$ contained two vertices in some original block $B_i$, then $\overline R$ would contain the two corresponding classes in $\overline B_i$, violating feasibility of $\overline R$. Therefore $R$ is feasible for $\mathcal{B}$ and $|R|=|\overline R|$, so $\alpha(\mathcal{B})\ge\alpha(\overline{\mathcal{B}})$. Combining the two inequalities proves equality.
\end{proof}

\begin{remark}[Cost of exact vertex-dominance closure]
Theorem~\ref{thm:doc-dominance}(i) is an exact reduction rule, but applying it exhaustively requires detecting support containment among vertex incidence sets:
\[
I(a)\subseteq I(b).
\]
Equivalently, one must compute the support antichain of the vertex family. If there are $N:=|V|$ active vertices, $M:=|\mathcal{B}|$ blocks, total incidence size
\[
I_{\mathrm{tot}}:=\sum_{v\in V}d(v),
\]
and maximum vertex degree $\Delta:=\max_v d(v)$, then a direct sorted-list comparison over all vertex pairs costs
\[
O(N^2\Delta)
\]
in the worst case. A bitset implementation instead costs
\[
O(N^2 M / w_{\mathrm{word}})
\]
bit operations and
\[
O(NM/w_{\mathrm{word}})
\]
words of storage. Inverted-index variants avoid comparing all pairs explicitly, but still require intersecting block posting lists for many possible dominators,
\[
\bigcap_{i\in I(a)}B_i,
\]
and can be output- or intersection-size dominated on dense block families. Thus full exact superset suppression is a valid maximum strong independent set preprocessing rule, but it is too expensive to use as a routine global pass at large scale.

By contrast, Theorem~\ref{thm:doc-dominance}(ii) is much cheaper: incidence-twin compression only groups equal supports, which can be implemented by hashing or sorting support lists and then verifying collisions. It is therefore used as shared preprocessing. The stricter containment rule from part~(i) is used later only in a bounded, layer-revealed form: Algorithm~\ref{alg:layered-superset-sharpening} applies partial dominance suppression to sharpen the closed-form certificate, rather than attempting full dominance closure of the maximum strong independent set instance.
\end{remark}

Part~(ii) can also be read as repeated application of part~(i): within each class $C$, choose one representative $r_C$ and delete every other $v\in C$, since $I(r_C)=I(v)$. The explicit quotient notation is still useful algorithmically because it records the deleted vertex ids for later materialization.

\begin{remark}[Algorithmic benefit of incidence-twin compression]
Incidence-twin compression weakly reduces the active vertex count and total incidence size:
\begin{equation}
|\overline V|\le |V|,
\qquad
\sum_i |\overline B_i|\le \sum_i |B_i|.
\end{equation}
If an incidence-twin class has size $k>1$ and common support size $d>0$, the total incidence size drops by at least $(k-1)d$. In addition to lowering memory and runtime, the quotient prevents identical-support multiplicities from inflating quantities such as $n_{\min}(B)$ that drive puncturing, covering, and greedy layered clustering priorities. It is therefore a preprocessing step for the algorithms and certificates below.
\end{remark}

The next characterization concerns blocks containing a vertex that appears nowhere else. When incidence-twin compression has been applied, degrees, block weights, and multiplicity counts are understood on the quotient family, with bars omitted for readability. Let
\begin{equation}
\mathcal{B}_1:=\{B\in\mathcal{B}: w(B)=1\},
\end{equation}
and for each $B\in\mathcal{B}_1$, choose one vertex $z_B\in B$ with $d(z_B)=1$.

\begin{theorem}[Canonical optimal form for weight-1 blocks]
\label{thm:w1-canonical}
There exists an optimal feasible set $O^\star$ such that
\begin{equation}
z_B\in O^\star \qquad \forall B\in\mathcal{B}_1.
\end{equation}
Equivalently, every weight-1 block can be rooted by a weight-1 vertex in an optimal solution.
\end{theorem}

\begin{proof}
Define $Z:=\{z_B: B\in\mathcal{B}_1\}$. Since $d(z_B)=1$, each $z_B$ belongs only to its own block $B$. Hence no two elements of $Z$ can collide in any block, so $Z$ is feasible.

Among all optimal feasible sets, choose $O$ maximizing $|O\cap Z|$. Suppose, for contradiction, that there exists $B\in\mathcal{B}_1$ with $z_B\notin O$.

If $O\cap B=\varnothing$, then $O\cup\{z_B\}$ is feasible (because $z_B$ appears only in $B$), which contradicts optimality of $O$.

Therefore $O\cap B=\{x_B\}$ for some $x_B\neq z_B$. Consider
\begin{equation}
O' := (O\setminus\{x_B\})\cup\{z_B\}.
\end{equation}
This set is feasible: replacing $x_B$ by $z_B$ affects only block $B$, where the occupancy remains one, and in every other block feasibility can only improve because $x_B$ is removed.
Also, $|O'|=|O|$, so $O'$ is optimal.
But now $|O'\cap Z|=|O\cap Z|+1$, contradicting the choice of $O$.

Hence $z_B\in O$ for every $B\in\mathcal{B}_1$. Taking $O^\star:=O$ proves the claim.
\end{proof}

In the running example, $B_1=\{a,u\}$ and $d(a)=1$. If an optimal solution used $u$ to cover the local conflict in $B_1$, replacing $u$ by $a$ cannot create any new conflict because $a$ appears nowhere else. Thus one may canonically root $B_1$ at $a$ before considering the higher-weight layers.

The preceding results characterize vertex dominance/twin compression and canonical weight-1 roots. The next rule gives a pivot-level canonicalization test built directly from Theorem~\ref{thm:doc-dominance}: after deleting vertices dominated by a chosen pivot, a simple block condition guarantees an optimal solution containing that pivot.

For a pivot vertex $a\in V$, define the vertices dominated by $a$, the incident-block neighborhood of $a$, and the residual blockers of $a$ by
\begin{equation}
D(a):=\{v\in V\setminus\{a\}: I(a)\subseteq I(v)\},
\qquad
\Gamma(a):=\bigcup_{i\in I(a)} B_i,
\qquad
R(a):=\Gamma(a)\setminus(\{a\}\cup D(a)).
\end{equation}
Thus $D(a)$ is exactly the set of vertices removable by Theorem~\ref{thm:doc-dominance}(i) using $a$ as dominator, while $R(a)$ is the set of remaining possible blockers of $a$ after that dominance removal.

\begin{theorem}[Pivot dominance closure and canonical inclusion]
\label{thm:pivot-dominance-closure}
For any pivot $a\in V$, the following hold.
\begin{enumerate}[leftmargin=1.6em,label=(\roman*),itemsep=0.15em]
    \item \textbf{(Closure under dominance removal)}
    \begin{equation}
    \alpha(\mathcal{B})=\alpha\!\left(\mathcal{B}-D(a)\right).
    \end{equation}
    \item \textbf{(Sufficient condition for an optimal set containing $a$)}
    If there exists a block $B^*\in\mathcal{B}$ such that
    \begin{equation}
    R(a)\subseteq B^*,
    \end{equation}
    then there exists an optimal feasible set $O^\star$ with $a\in O^\star$.
\end{enumerate}
\end{theorem}

\begin{proof}
For (i), remove vertices in $D(a)$ one-by-one. For each removed vertex $v\in D(a)$, one has
\(
I(a)\subseteq I(v),
\)
so Theorem~\ref{thm:doc-dominance}(i) applies with pivot $a$, and the optimum value is preserved at each step. Since $D(a)$ is finite, the final reduced instance satisfies
\(
\alpha(\mathcal{B})=\alpha(\mathcal{B}-D(a)).
\)

For (ii), let $\mathcal{B}'$ denote the reduced instance after removing $D(a)$. By (i), $\alpha(\mathcal{B}')=\alpha(\mathcal{B})$. Take an optimal feasible set $O$ for $\mathcal{B}'$.
If $a\in O$, we are done.

Assume $a\notin O$. Since $O$ is feasible and $B^*$ is a block, $|O\cap B^*|\le 1$.

If $O\cap B^*=\varnothing$, then $O\cap R(a)=\varnothing$ because $R(a)\subseteq B^*$. Adding $a$ keeps feasibility: any potential conflict with $a$ must come from a block containing $a$, hence from $\Gamma(a)$, but $D(a)$ was removed and $O$ has no element in $R(a)$. Thus $O\cup\{a\}$ is feasible, contradicting optimality.

So $O\cap B^*=\{x\}$ for some $x$.
If $x\in B^*\setminus\Gamma(a)$, then $x$ is in no block containing $a$, so the same argument shows $O\cup\{a\}$ is feasible, again contradicting optimality.
Therefore necessarily
\[
x\in R(a)\subseteq B^*.
\]
Define
\begin{equation}
O':=(O\setminus\{x\})\cup\{a\}.
\end{equation}
Any block containing $a$ can intersect $O$ only at $x$: its members lie in $\Gamma(a)$; elements of $D(a)$ are absent in $\mathcal{B}'$; and $|O\cap R(a)|=1$ because $R(a)\subseteq B^*$ and $O\cap B^*=\{x\}$. Hence $O'$ is feasible and $|O'|=|O|=\alpha(\mathcal{B}')$.
So $\mathcal{B}'$ has an optimal set containing $a$, and by value equivalence this is also optimal for the original instance.
\end{proof}

For the equality case below, write
\[
\mathcal{B}' := \mathcal{B}-D(a)
\]
for the pivot-reduced instance.

\begin{corollary}[Equality case of pivot canonicalization]
If Theorem~\ref{thm:pivot-dominance-closure}(ii) holds with
\[
B^*=R(a),
\]
then every optimal feasible set $O$ for $\mathcal{B}'$ with $a\notin O$ satisfies
\[
|O\cap R(a)|=1.
\]
In particular, swapping that unique element with $a$ yields an optimal feasible set containing $a$.
\end{corollary}

\begin{proof}
By feasibility and $B^*=R(a)$, one has $|O\cap R(a)|\le 1$. If the intersection were empty, then $O\cup\{a\}$ would be feasible in $\mathcal{B}'$ (same argument as in the theorem proof), contradicting optimality. Hence the intersection size is exactly one, and the swap conclusion follows from Theorem~\ref{thm:pivot-dominance-closure}(ii).
\end{proof}

The equality case also has a useful downstream consequence for the closed-form block-weight bound introduced in the next section. In the equality case $B^*=R(a)$, one may canonically keep $a$, delete $D(a)$, and then apply Theorem~\ref{thm:block-weight-upper} to the reduced instance $\mathcal{B}-D(a)$. This yields a weakly smaller computable upper bound than applying Theorem~\ref{thm:block-weight-upper} directly to the original family, although the improvement is only local and typically marginal.

\begin{example}[Dominance closure and equality-case swapping]
\label{ex:dominance-closure}
Let
\[
\mathcal{B}=\bigl\{
\{a,d,u,v\},\{a,d\},\{u,v\},\{u,x\},\{v,y\},\{x,y\}
\bigr\}.
\]
All vertices have incidence degree at least two. For pivot $a$,
\[
I(a)=\{1,2\},\qquad I(d)=\{1,2\},
\]
so $d\in D(a)$ and dominance removal deletes $d$ without changing the optimum. In the reduced instance, $R(a)=\{u,v\}$, and this set is exactly the block $B^*=\{u,v\}$. Theorem~\ref{thm:pivot-dominance-closure}(ii) therefore guarantees an optimal feasible set containing $a$.

The equality-case corollary describes the only obstruction to already adding $a$. Any optimum excluding $a$ must contain exactly one vertex from $\{u,v\}$. For example, $\{u,y\}$ is feasible and optimal in the reduced instance; swapping the unique element $u\in R(a)$ with $a$ gives the equally large optimum $\{a,y\}$.
\end{example}

This dominance theorem also explains why we do not state separate private-shadow certificates. If a blocker $u$ of a pivot $a$ has a shadow $z$ outside $\Gamma(a)$ with
\[
I(z)\subseteq I(u)\setminus I(a),
\]
then $I(z)\subseteq I(u)$, so $u$ is already removable by Theorem~\ref{thm:doc-dominance}(i). Thus such shadow arguments are best viewed as dominance preprocessing followed by the pivot canonicalization above, not as independent structural rules.

The final characterization separates structural optimality from algorithmic attainability. It identifies an incidence-block pattern whose optimum is certified exactly, regardless of whether a particular greedy chronology finds the displayed optimal set.

We first fix the reduced-instance notation used in the certificate. Perform vertex-level dominance suppression by choosing a set $D\subseteq V$ such that every $v\in D$ has a surviving witness $r(v)\in V\setminus D$ with
\[
I(r(v))\subseteq I(v).
\]
Define the dominance-reduced block family by
\[
\mathcal{B}^{-}
:=
\{B\setminus D:\ B\in\mathcal{B},\ B\setminus D\neq\varnothing\},
\]
index it as
\[
\mathcal{B}^{-}=\{B_i^-:i\in\mathcal{I}^{-}\},
\]
and let $I_-(v)$, $d_-(v)$, and $w_-(\cdot)$ denote incidence sets, degrees, and block weights computed inside $\mathcal{B}^{-}$.
Suppose the block-index set $\mathcal{I}^{-}$ is partitioned into nonempty blocks
\[
\mathcal{J}_1,\dots,\mathcal{J}_m,
\]
and for each $t\in\{1,\dots,m\}$ there exists a surviving vertex $z_t\in V\setminus D$ such that
\[
I_-(z_t)=\mathcal{J}_t.
\]
Set
\[
Z:=\{z_1,\dots,z_m\}.
\]
The following certificate says that, after support-superset suppression, these block representatives already form an optimum.

\begin{theorem}[Dominance-reduced incidence-block optimality certificate]
\label{thm:incidence-block-optimality}
Assume that every reduced block in the same block has block-size weight:
\[
w_-(B_i^-)=|\mathcal{J}_t|
\qquad
\forall t\in\{1,\dots,m\},\ \forall i\in\mathcal{J}_t.
\]
Then $Z$ is a maximum feasible set for the reduced family and for the original family. In particular,
\begin{equation}
\alpha(\mathcal{B})=m.
\end{equation}
\end{theorem}

\begin{proof}
For every $v\in D$, the witness $r(v)$ survives the suppression pass and satisfies $I(r(v))\subseteq I(v)$. Deleting the vertices of $D$ one by one is therefore value-preserving by Theorem~\ref{thm:doc-dominance}(i), and empty blocks impose no feasibility constraint. Hence
\[
\alpha(\mathcal{B}^{-})=\alpha(\mathcal{B}).
\]

In $\mathcal{B}^{-}$, each $z_t$ has support $\mathcal{J}_t$. Hence $Z$ is feasible: if $i\in\mathcal{J}_t$, then $z_s\in B_i^-$ iff $i\in I_-(z_s)=\mathcal{J}_s$, which happens only for $s=t$ because the blocks form a partition.

Now let $O$ be any feasible set for $\mathcal{B}^{-}$. Assign one unit of charge to each $v\in O$ and distribute it evenly over the $d_-(v)$ reduced blocks containing $v$. Each block receives charge from at most one retained vertex. If $i\in\mathcal{J}_t$ and $O\cap B_i^-=\{v\}$, then
\[
d_-(v)\ge w_-(B_i^-)=|\mathcal{J}_t|,
\]
so the reduced block receives at most $1/|\mathcal{J}_t|$ charge. Therefore the total charge received by reduced blocks in block $\mathcal{J}_t$ is at most
\[
\sum_{i\in\mathcal{J}_t}\frac{1}{|\mathcal{J}_t|}=1.
\]
Summing over the $m$ blocks gives $|O|\le m$. Since $Z$ is feasible in $\mathcal{B}^{-}$ and $|Z|=m$, it is maximum in $\mathcal{B}^{-}$. The dominance deletion gives $\alpha(\mathcal{B})=\alpha(\mathcal{B}^{-})=m$. Because $Z\cap D=\varnothing$, the same set $Z$ is feasible for the original family $\mathcal{B}$ as well.
\end{proof}

\begin{runningexamplepart}[Running example: dominance-reduced incidence blocks]
\label{ex:running-incidence-block}
Return to the running example in Example~\ref{ex:running-example-family}. First perform the support-superset suppression
\[
D=\{u,f,g\},
\]
because
\[
\begin{aligned}
I(a)&=\{1\}\subsetneq I(u)=\{1,2\},\\
I(b)&=\{2,3\}\subsetneq I(f)=\{2,3,9\},\\
I(e)&=\{4,5\}\subsetneq I(g)=\{4,5,10\}.
\end{aligned}
\]
The reduced family $\mathcal{B}^{-}$ keeps the same block labels. In $\mathcal{B}^{-}$, take the same block partition later used by the greedy trace:
\[
\mathcal{J}_1=\{1\},\qquad
\mathcal{J}_2=\{2,3\},\qquad
\mathcal{J}_3=\{4,5\},\qquad
\mathcal{J}_4=\{6,7,8\},\qquad
\mathcal{J}_5=\{9,10,11,12\},
\]
with representatives
\[
z_1=a,\qquad z_2=b,\qquad z_3=e,\qquad z_4=c,\qquad z_5=r.
\]
The reduced family satisfies the block-size weight condition:
\[
\begin{aligned}
w_-(B_1)&=1,\qquad
w_-(B_2)=\cdots=w_-(B_5)=2,\\
w_-(B_6)&=w_-(B_7)=w_-(B_8)=3,\qquad
w_-(B_9)=\cdots=w_-(B_{12})=4.
\end{aligned}
\]
Thus Theorem~\ref{thm:incidence-block-optimality} certifies
\[
Z=\{a,b,e,c,r\}
\]
as a maximum feasible set and gives
\[
\alpha(\mathcal{B})=5.
\]
The unreduced closed-form sum is larger, $U(\mathcal{B})=31/6$, because $f$ and $g$ artificially lower the weights of $B_9$ and $B_{10}$ from $4$ to $3$. The dominance-reduced certificate removes exactly this distortion. Section~5 shows that the greedy layered clustering algorithm realizes the same suppression intrinsically: the low-weight roots $a,b,e$ absorb $u,f,g$ before the higher-weight layers are processed.
\end{runningexamplepart}

Except for incidence-twin compression and the weight-1 canonical form, which are used below as preprocessing and first certificate steps, these exact characterization rules are optional tools: they can shrink or partially canonicalize an instance without changing the optimum, but the scalable certificates and algorithms below do not rely on the more expensive dominance-closure preprocessing being applied.

\section{Bounds and Certificates for Maximum Strong Independent Sets}

\subsection{Closed-Form and Low-Weight Certificates}

Exact characterizations are useful when their hypotheses hold; at billion-vertex scale we also need certificates that are cheap to compute on arbitrary block families. With block weight already defined, the first certificate is the closed-form block-weight sum.

\begin{definition}[Closed-form upper-bound value]
For any block family $\mathcal{F}$ with intrinsic degree function $d_{\mathcal{F}}(\cdot)$ and induced block weights
\begin{equation}
w_{\mathcal{F}}(B):=\min_{v\in B} d_{\mathcal{F}}(v),
\end{equation}
define
\begin{equation}
U(\mathcal{F}) := \sum_{B\in\mathcal{F}} \frac{1}{w_{\mathcal{F}}(B)}.
\end{equation}
In particular, for the ambient family $\mathcal{B}$ this is exactly
\begin{equation}
U(\mathcal{B})=\sum_{B\in\mathcal{B}} \frac{1}{w(B)}.
\end{equation}
\end{definition}

\begin{theorem}[Closed-form upper bound]
\label{thm:block-weight-upper}
For every feasible retained set $R\subseteq V$,
\begin{equation}
|R| \le U(\mathcal{B}).
\end{equation}
\end{theorem}

\begin{proof}
Assign one unit of charge to each retained vertex $v\in R$, and distribute it evenly among the $d(v)$ blocks containing $v$. Each incident block receives charge $1/d(v)$. Since $R$ is feasible, each block receives charge from at most one retained vertex. If $v\in R\cap B$, then by definition $w(B)\le d(v)$, hence $1/d(v)\le 1/w(B)$. Summing over all blocks, the total received charge is at most $U(\mathcal{B})$. On the other hand, each retained vertex contributes total charge 1, so the total charge equals $|R|$.
\end{proof}

The intuitive interpretation is straightforward. A vertex with large $d(v)$ is globally entangled: choosing it consumes capacity across many blocks. A block containing a degree-1 or degree-2 vertex is therefore structurally sharp and should be resolved early.

\begin{runningexamplepart}[Running example: closed-form bound]
\label{ex:running-example-bound}
Applying Theorem~\ref{thm:block-weight-upper} directly to the unreduced running example gives
\[
U(\mathcal{B})
=
1+4\cdot\frac12+5\cdot\frac13+2\cdot\frac14
=
\frac{31}{6}.
\]
The dominance-reduced certificate is sharper. The representatives
\[
z_1=a,\quad z_2=b,\quad z_3=e,\quad z_4=c,\quad z_5=r
\]
have supports
\[
\{1\},\quad \{2,3\},\quad \{4,5\},\quad \{6,7,8\},\quad \{9,10,11,12\},
\]
and their dominated support-superset vertices include
\[
u,\ f,\ g
\quad\text{with}\quad
I(a)\subsetneq I(u),\ I(b)\subsetneq I(f),\ I(e)\subsetneq I(g).
\]
After deleting $u,f,g$, Theorem~\ref{thm:incidence-block-optimality} applies to the reduced family and gives
\[
U(\mathcal{B}^{-})
=
1+4\cdot\frac12+3\cdot\frac13+4\cdot\frac14
=
5.
\]
Thus the optimal value is already certified as
\[
\alpha(\mathcal{B})=5.
\]
Section~5 will return to this same instance and show that the greedy layered clustering algorithm performs the same first-order superset suppression and attains this value exactly.
\end{runningexamplepart}

The canonical optimal form for weight-1 blocks (Theorem~\ref{thm:w1-canonical}) yields the following computable preclusion bound. Let
\begin{equation}
\mathcal{B}_1:=\{B\in\mathcal{B}:w(B)=1\},\qquad
\mathcal{B}_{>1}:=\mathcal{B}\setminus\mathcal{B}_1,\qquad
V_1:=\bigcup_{B\in\mathcal{B}_1} B.
\end{equation}
For each $B\in\mathcal{B}_{>1}$ define the residual block
\begin{equation}
B^\circ:=B\setminus V_1,
\end{equation}
and keep only nonempty residual blocks
\begin{equation}
\mathcal{B}^\circ:=\{B^\circ: B\in\mathcal{B}_{>1},\ B^\circ\neq\varnothing\}.
\end{equation}
On $\mathcal{B}^\circ$, residual vertex degrees equal the original degrees of the surviving vertices, so the residual block weights are
\begin{equation}
w^\circ(B^\circ):=\min_{v\in B^\circ} d(v),
\end{equation}
and
\begin{equation}
U(\mathcal{B}^\circ)=\sum_{B^\circ\in\mathcal{B}^\circ}\frac{1}{w^\circ(B^\circ)}.
\end{equation}

\begin{corollary}[Weight-1 preclusion bound]
\label{cor:w1-elim-upper}
Every feasible retained set $R$ satisfies
\begin{equation}
|R|
\le
|\mathcal{B}_1|
\;+\;
U(\mathcal{B}^\circ).
\end{equation}
Moreover,
\begin{equation}
|\mathcal{B}_1|
\;+\;
U(\mathcal{B}^\circ)
\le
U(\mathcal{B}),
\end{equation}
so this is never weaker than Theorem~\ref{thm:block-weight-upper}.
\end{corollary}

\begin{proof}
Split $R$ as $R=(R\cap V_1)\,\cup\,(R\setminus V_1)$.
For the first part, each $x\in R\cap V_1$ belongs to at least one block in $\mathcal{B}_1$, hence
\begin{equation}
|R\cap V_1|
\le
\sum_{B\in\mathcal{B}_1}|R\cap B|
\le
|\mathcal{B}_1|.
\end{equation}

Now consider $R_0:=R\setminus V_1\subseteq V\setminus V_1$. For every $B\in\mathcal{B}_{>1}$,
\begin{equation}
|R_0\cap B^\circ|=|R_0\cap B|\le 1,
\end{equation}
so $R_0$ is feasible for residual blocks $\mathcal{B}^\circ$.
If $v\in V\setminus V_1$, then $v$ is in no block of $\mathcal{B}_1$, hence its degree over $\mathcal{B}^\circ$ is exactly $d(v)$, and the block-weight theorem applies with weights $w^\circ$:
\begin{equation}
|R_0|\le U(\mathcal{B}^\circ).
\end{equation}
Adding both parts gives the claimed refined bound.

For the second inequality, $\sum_{B\in\mathcal{B}_1}1/w(B)=|\mathcal{B}_1|$.
Also for each $B^\circ$ induced from $B\in\mathcal{B}_{>1}$, one has
\begin{equation}
w^\circ(B^\circ)=\min_{v\in B^\circ}d(v)\ge \min_{v\in B}d(v)=w(B),
\end{equation}
thus $1/w^\circ(B^\circ)\le 1/w(B)$, and summing yields the result.
\end{proof}

The next corollary gives a complementary achievability certificate for residual weight-2 structure. Unlike the weight-1 case, weight 2 alone does not force a canonical root in each block. However, the degree-2 vertices define an ordinary graph-matching problem over block incidences in the sense of Edmonds' blossom algorithm \cite{edmonds1965}. The size of such a matching is always a feasible retained count, and it certifies optimality exactly when it saturates the integer closed-form bound.

For this certificate, let $\mathcal{C}$ be a residual block family on active universe
\[
V_{\mathcal{C}}:=\bigcup_{B\in\mathcal{C}}B,
\]
with intrinsic degrees $d_{\mathcal{C}}(\cdot)$ and intrinsic block weights $w_{\mathcal{C}}(\cdot)$. Construct a graph $G_2(\mathcal{C})$ whose vertices are the blocks in $\mathcal{C}$ and whose edges are the degree-2 vertices: for each vertex $v$ with $d_{\mathcal{C}}(v)=2$ and incident blocks $B_i,B_j\in\mathcal{C}$, add an edge $(B_i,B_j)$ labeled by $v$.
Assume
\[
w_{\mathcal{C}}(B)=2
\qquad
\forall B\in\mathcal{C},
\]
and let $M$ be any matching in $G_2(\mathcal{C})$.

\begin{corollary}[Saturated weight-2 matching certificate]
\label{cor:w2-matching-certificate}
The vertices labeling the matched edges form a feasible retained set for $\mathcal{C}$, and
\begin{equation}
|M|
\le
\alpha(\mathcal{C})
\le
\left\lfloor\frac{|\mathcal{C}|}{2}\right\rfloor
=
\left\lfloor U(\mathcal{C})\right\rfloor.
\end{equation}
Consequently, if $|M|=\lfloor |\mathcal{C}|/2\rfloor$, then the matched vertices form an optimal retained set for $\mathcal{C}$.
\end{corollary}

\begin{proof}
Since every block in $\mathcal{C}$ has intrinsic weight $2$, Theorem~\ref{thm:block-weight-upper} applied to $\mathcal{C}$ gives
\[
\alpha(\mathcal{C})
\le
U(\mathcal{C})
=
\sum_{B\in\mathcal{C}}\frac{1}{2}
=
\frac{|\mathcal{C}|}{2}.
\]
The left side is an integer, so
\[
\alpha(\mathcal{C})
\le
\left\lfloor\frac{|\mathcal{C}|}{2}\right\rfloor.
\]
Now retain the vertex labeling each matched edge of $M$. Because $M$ is a matching, no two selected degree-2 vertices share an incident block. Hence every block in $\mathcal{C}$ contains at most one selected vertex, so the selected vertices are feasible. This gives $|M|\le\alpha(\mathcal{C})$. If $|M|=\lfloor|\mathcal{C}|/2\rfloor$, the feasible matching construction attains the integer upper bound and is therefore optimal.
\end{proof}

As a minimal saturated example, take four blocks arranged in a cycle by degree-2 vertices:
\[
C_1=\{x_{12},x_{41}\},\quad
C_2=\{x_{12},x_{23}\},\quad
C_3=\{x_{23},x_{34}\},\quad
C_4=\{x_{34},x_{41}\}.
\]
Every block has intrinsic weight $2$, so $U(\mathcal{C})=2$. The matching $\{(C_1,C_2),(C_3,C_4)\}$, labeled by $x_{12}$ and $x_{34}$, saturates the bound and certifies that $\{x_{12},x_{34}\}$ is optimal.

This certificate is deliberately only sufficient. It can hold in the presence of higher-degree vertices, but it certifies optimality only when the degree-2 layer already saturates the integer block-weight bound. It is therefore not a weight-2 analogue of Corollary~\ref{cor:w1-elim-upper}: there is no unconditional weight-2 preclusion pass, no automatic additive term to peel off, and no guaranteed bound improvement from the mere fact that a block has weight 2. Rather, Corollary~\ref{cor:w2-matching-certificate} is an achievability check showing that, in the special saturated-matching case, the closed-form block-weight bound is tight.

The closed-form and weight-1 certificates above are evaluated on a fixed active block family. The next subsection treats them as reusable primitives inside a residual search: support-superset sharpening and iterative puncturing repeatedly modify the active family, recompute intrinsic weights, and keep the smallest resulting valid upper bound $U_{\mathrm{punct}}$.

\subsection{Layered Superset Sharpening and Puncturing}

The cost discussion after Theorem~\ref{thm:doc-dominance} explains why we do not run full exact superset suppression as a global preprocessing pass. The next algorithm uses the same dominance principle more selectively: it looks only at supports revealed by current block-weight vertices, deletes the higher-degree support-superset vertices certified by those witnesses, and uses the resulting weight increases to sharpen the closed-form certificate. We abbreviate this procedure as LSS in the empirical tables.

\begin{algorithm}[Layered Superset Sharpening]
\label{alg:layered-superset-sharpening}
Given an active block family $\mathcal{F}=\{F_1,\dots,F_m\}$ whose closed-form certificate will be evaluated after the surrounding routine has exhausted its weight-1 preclusion step, perform the following pass.
\begin{enumerate}[leftmargin=1.6em,label=\arabic*.,itemsep=0.15em]
    \item First perform incidence-twin compression on the remaining active family, then compute vertex degrees
    \[
    d_{\mathcal{F}}(v)=|\{i:v\in F_i\}|
    \]
    and block weights $w_{\mathcal{F}}(F_i)=\min_{v\in F_i}d_{\mathcal{F}}(v)$.
    \item Process block-weight layers from smallest to largest. In layer $w$, build a fresh local incidence map
    \[
    J_w(v):=\{i:v\in F_i,\ w_{\mathcal{F}}(F_i)=w\}.
    \]
    Maintain the accumulated support
    \[
    P(v):=\bigcup_{\lambda\le w}J_\lambda(v)
    \]
    over layers already scanned.
    \item Whenever a vertex $a$ satisfies
    \[
    |P(a)|=d_{\mathcal{F}}(a)=w,
    \]
    its full live support $I_{\mathcal{F}}(a)=P(a)$ has been revealed exactly at its own block-weight layer. Such a vertex is a block-weight witness for that support.
    \item Intersect the blocks in this revealed support,
    \[
    C(a):=\bigcap_{i\in P(a)}F_i,
    \]
    and delete every active $b\in C(a)$ with $d_{\mathcal{F}}(b)>w$ from the residual certificate family. Exact identity/equal-support cases are handled by the incidence-twin compression step at the start of the pass.
    \item After a support $P(a)$ has been processed, remove $a$ from the accumulated map to save memory. After a nonempty deletion pass, rebuild incidence twins, recompute degrees and block weights, and repeat until no block-weight witness deletes a vertex.
\end{enumerate}
\end{algorithm}

For the next result and its cost discussion, let
\[
I_{\mathcal{F}}:=\sum_{F\in\mathcal{F}}|F|
\]
denote the active incidence size for one sharpening pass, and let $T_{\cap}$ denote the total cost of the block intersections performed for newly revealed supports.
Let $\mathcal{F}^{\sharp}$ denote the terminal active family after Algorithm~\ref{alg:layered-superset-sharpening} stops.
For a single nonempty deletion pass, write $\mathcal{F}'$ for the family obtained after deleting the certified support-superset vertices and discarding any empty blocks.

\begin{proposition}[Correctness and fixed point of layered superset sharpening]
\label{prop:layered-superset-sharpening}
Every support-superset deletion made by Algorithm~\ref{alg:layered-superset-sharpening} preserves $\alpha(\mathcal{F})$. Moreover, if $\mathcal{F}'$ is obtained from $\mathcal{F}$ by one nonempty deletion pass, then
\[
\sum_{F'\in\mathcal{F}'}\frac{1}{w_{\mathcal{F}'}(F')}
\le
\sum_{F\in\mathcal{F}}\frac{1}{w_{\mathcal{F}}(F)}.
\]
In $\mathcal{F}^{\sharp}$, no active block $F$, block-weight witness $a\in F$, and active vertex $b$ satisfy
\[
d_{\mathcal{F}^{\sharp}}(a)=w_{\mathcal{F}^{\sharp}}(F),
\qquad
I_{\mathcal{F}^{\sharp}}(a)\subsetneq I_{\mathcal{F}^{\sharp}}(b).
\]
\end{proposition}

\begin{proof}
Consider a deletion made when the support of a block-weight witness $a$ is revealed in layer $w$. The algorithm has
\[
d_{\mathcal{F}}(a)=|P(a)|=w,\qquad P(a)=I_{\mathcal{F}}(a),
\]
and it deletes only vertices $b$ in
\[
C(a)=\bigcap_{i\in P(a)}F_i.
\]
Therefore $b$ belongs to every block containing $a$, so $I_{\mathcal{F}}(a)\subseteq I_{\mathcal{F}}(b)$. Since the algorithm deletes only $b$ with $d_{\mathcal{F}}(b)>w=d_{\mathcal{F}}(a)$, the containment is strict, and Theorem~\ref{thm:doc-dominance}(i) permits deleting $b$ without changing $\alpha(\mathcal{F})$. Repeating the argument over all deletions and rebuild passes proves preservation of the optimum by induction. Equal-support vertices are already handled by the incidence-twin quotient before each pass.

Removing vertices from blocks cannot decrease the degree of any surviving vertex and can only increase, or leave unchanged, the minimum surviving degree in each nonempty block. Hence each surviving block term is weakly decreased when recomputed in $\mathcal{F}'$; empty blocks contribute nothing. Thus the displayed closed-form sum weakly decreases.

It remains to prove the stated fixed-point property. Let $\mathcal{F}^{\sharp}$ be the active family after the final rebuild, incidence-twin compression, and sharpening pass. Suppose for contradiction that an active block $F$, a block-weight witness $a\in F$, and an active vertex $b$ satisfy
\[
a\in F,\qquad
d_{\mathcal{F}^{\sharp}}(a)=w_{\mathcal{F}^{\sharp}}(F),
\qquad
I_{\mathcal{F}^{\sharp}}(a)\subsetneq I_{\mathcal{F}^{\sharp}}(b).
\]
When the final pass reaches layer $w_{\mathcal{F}^{\sharp}}(F)=d_{\mathcal{F}^{\sharp}}(a)$, the full support of $a$ is revealed, and $b$ belongs to every block in that support. Hence $b\in C(a)$ and $d_{\mathcal{F}^{\sharp}}(b)>d_{\mathcal{F}^{\sharp}}(a)$, so the pass would delete $b$. This contradicts termination.

\end{proof}

Equivalently, no higher-degree support superset remains for the block-weight supports inspected by the pass. This is a certificate-sharpening property, not a full dominance closure over all vertices: support containments whose smaller support is not a current block-weight witness may remain because they do not affect the first-order closed-form certificate in the same way.

The implementation cost is also local. A pass first scans all active incidences to compute degrees and block weights, then scans each block once in its layer to update the accumulated supports $P(\cdot)$; this contributes $O(I_{\mathcal{F}})$. For every newly revealed support, the algorithm intersects its incident blocks, contributing total work $T_{\cap}$ across the pass. Applying a nonempty deletion set requires one additional linear scan of the active block lists. Thus one pass costs $O(I_{\mathcal{F}}+T_{\cap})$, with peak auxiliary memory given by the accumulated not-yet-fully-revealed support frontier plus temporary intersection sets.

By contrast, a brute-force dominance oracle would materialize full supports for all vertices and test $I(a)\subseteq I(b)$ over many vertex pairs, costing $\Omega(|V|^2)$ support comparisons in the worst case, even before accounting for support lengths. Algorithm~\ref{alg:layered-superset-sharpening} avoids this global pair table: it only performs intersections for supports as they become fully revealed by the layer order.

On the running example, the preliminary weight-1 preclusion removes the block $B_1$ and clusters away $u$ with the root $a$. Algorithm~\ref{alg:layered-superset-sharpening} then starts on the residual family. In layer $2$, the full supports of the block-weight witnesses $b$ and $e$ are revealed:
\[
I(b)=\{2,3\},\qquad I(e)=\{4,5\}.
\]
The intersections $B_2\cap B_3$ and $B_4\cap B_5$ contain the higher-degree vertices $f$ and $g$, respectively. Hence the pass deletes $f$ and $g$, raising the residual weights of $B_9$ and $B_{10}$ from their apparent weight $3$ to their reduced weight $4$. This is the intended role of the algorithm: it sharpens the certificate where block-weight witnesses expose a support-superset distractor, without attempting a global dominance closure.

Beyond the canonical weight-1 layer, the closed-form bound can be further sharpened by puncturing the block family: delete selected blocks, recompute the incidence weights on the residual family, and keep the smallest valid certificate encountered. Puncturing does not require every surviving block to keep its original weight. The validity comes from monotonicity of the feasible region under deletion, while the tightening comes from searching for a residual family whose recomputed certificate is smaller.

For the certificate below, consider a residual instance obtained from $\mathcal{B}$ by deleting blocks and by applying weight-1 preclusion to fix $q$ roots and remove their clustered vertices from all remaining blocks. Let $\mathcal{C}$ be the remaining nonempty block family, and assume every still-active vertex is incident to at least one block of $\mathcal{C}$. Its intrinsic degrees, block weights, and closed-form value are
\begin{equation}
d_{\mathcal{C}}(v):=\bigl|\{B\in\mathcal{C}:v\in B\}\bigr|,
\qquad
w_{\mathcal{C}}(B):=\min_{v\in B}d_{\mathcal{C}}(v),
\end{equation}
and
\begin{equation}
U(\mathcal{C}):=\sum_{B\in\mathcal{C}}\frac{1}{w_{\mathcal{C}}(B)}.
\end{equation}
If the residual instance is additionally processed by repeated weight-1 preclusion, write $\mathcal{C}^{\circ}$ for the final nonempty residual family.

\begin{proposition}[Punctured-family upper-bound certificate]
\label{prop:punctured-certificate}
\begin{equation}
\alpha(\mathcal{B})\le q+\alpha(\mathcal{C})\le q+U(\mathcal{C}).
\end{equation}
In particular,
\begin{equation}
\alpha(\mathcal{B})\le q+U(\mathcal{C}^{\circ}).
\end{equation}
\end{proposition}

\begin{proof}
Each weight-1 preclusion step charges at most one retained vertex to the precluded block, so all precluded parts contribute at most $q$ retained vertices. After these charged vertices are removed, deleting blocks only relaxes the remaining feasibility constraints. Hence the uncharged part of any feasible retained set for the original family is feasible for the residual family $\mathcal{C}$ on its active universe, giving
\[
\alpha(\mathcal{B})\le q+\alpha(\mathcal{C}).
\]
Because every still-active vertex is incident to at least one block of $\mathcal{C}$, Theorem~\ref{thm:block-weight-upper} applies to the residual instance with degrees and weights computed inside $\mathcal{C}$, yielding $\alpha(\mathcal{C})\le U(\mathcal{C})$. The final statement is the same argument applied after repeated weight-1 preclusion.
\end{proof}

Proposition~\ref{prop:punctured-certificate} turns bound sharpening into a certificate-search problem. The next algorithm is the practical iterative search used in our implementation. It processes current block-weight layers from high to low and, within each layer, dynamically chooses the block with the smallest current minimum-degree count. It accepts punctures that do not increase the current certificate, including zero-gain punctures that expose new weight-1 blocks; those blocks are precluded at the next rebuild and can unlock a strictly smaller certificate. The returned value is the \emph{iterative puncturing bound}, denoted $U_{\mathrm{punct}}$, namely the smallest valid certificate recorded over all outer rebuild-and-puncture rounds.

\begin{algorithm}[Iterative Greedy Layered Puncturing Algorithm]
\label{alg:greedy-puncturing}
Initialize the live family $\mathcal{C}\leftarrow\mathcal{B}$, the fixed weight-1 contribution $q\leftarrow 0$, and the best iterative certificate $U_{\mathrm{punct}}\leftarrow U(\mathcal{B})$.
For a current residual family $\mathcal{C}$, write
\[
n_{\min,\mathcal{C}}(B)
:=
\bigl|\{v\in B:d_{\mathcal{C}}(v)=w_{\mathcal{C}}(B)\}\bigr|
\]
for the number of minimum-degree vertices in block $B$.

Repeat the following outer rounds until no block is punctured in a full round.
\begin{enumerate}[leftmargin=1.6em,label=\arabic*.,itemsep=0.15em]
    \item \textbf{Rebuild, preclude weight-1 blocks, and sharpen the certificate.} Recompute $d_{\mathcal{C}}(\cdot)$ and $w_{\mathcal{C}}(\cdot)$, compress incidence twins, and while there is a block $B\in\mathcal{C}$ with $w_{\mathcal{C}}(B)=1$, choose a degree-1 root in $B$, add one to $q$, remove the vertices of $B$ from all live blocks, delete empty blocks, and rebuild the affected degrees and weights. After the current weight-1 wave is exhausted, apply Algorithm~\ref{alg:layered-superset-sharpening} to the residual certificate family. If sharpening deletes vertices, rebuild incidence twins, degrees, and block weights before recording the certificate. Weight-1 blocks are not newly created by this vertex-deletion pass; they are exposed by block punctures and handled at the next rebuild.
    \item \textbf{Record the current certificate.} Update
    \begin{equation}
    U_{\mathrm{punct}}\leftarrow
    \min\left\{U_{\mathrm{punct}},\ q+\sum_{B\in\mathcal{C}}\frac{1}{w_{\mathcal{C}}(B)}\right\}.
    \end{equation}
    If $\mathcal{C}=\varnothing$, stop.
    \item \textbf{Process layers from high to low.} Maintain a global vertex-to-live-block index and current block weights for the residual family. For current weights $w$ in decreasing order, initialize the active layer
    \[
    \mathcal{L}_w:=\{B\in\mathcal{C}:w_{\mathcal{C}}(B)=w\}
    \]
    with scores $n_{\min,\mathcal{C}}(B)$. Repeatedly choose a block $B\in\mathcal{L}_w$ with minimum current $n_{\min,\mathcal{C}}(B)$, breaking ties arbitrarily. Use the global index to identify every same- or higher-weight block whose weight may change after puncturing $B$, and compute the exact local certificate change
    \begin{equation}
    \Delta(B):=
    \left(q+\sum_{C\in\mathcal{C}}\frac{1}{w_{\mathcal{C}}(C)}\right)
    -
    \left(q+\sum_{C\in\mathcal{C}\setminus\{B\}}\frac{1}{w_{\mathcal{C}\setminus\{B\}}(C)}\right),
    \end{equation}
    updating only blocks incident to vertices of $B$.
    \item \textbf{Puncture non-worsening or weight-1-enabling blocks.} Delete $B$ if $\Delta(B)\ge 0$, or if deleting $B$ creates a new weight-1 block that will be precluded in the next rebuild. After a deletion, update the affected vertex degrees, block weights, and active-layer scores through the global index; if a block's weight decreases below $w$, move it to the corresponding later layer. At the end of each layer, discard index entries for vertices that cannot affect any future processed layer.
\end{enumerate}
Output the iterative puncturing bound $U_{\mathrm{punct}}$.
\end{algorithm}

\begin{remark}[Layered puncturing intuition]
A puncture may lower the weight of some surviving blocks, so it is not justified by preserving the old closed-form certificate. Instead, it is justified by recomputing the certificate on the residual family and retaining the best value seen. The non-worsening condition handles immediately useful punctures, while the weight-1-enabling condition permits zero-gain moves that expose canonical roots after the next rebuild. Empirically, these zero-gain moves are crucial: they often create the next wave of weight-1 preclusions and therefore produce a much tighter final certificate.
\end{remark}

\begin{corollary}[Validity and termination of the iterative puncturing bound]
\label{cor:greedy-puncturing-valid}
The iterative puncturing bound $U_{\mathrm{punct}}$ returned by Algorithm~\ref{alg:greedy-puncturing} is a valid upper bound:
\begin{equation}
\alpha(\mathcal{B})\le U_{\mathrm{punct}}.
\end{equation}
Moreover, the algorithm terminates after finitely many punctures and preclusion steps.
\end{corollary}

\begin{proof}
Every certificate recorded by the algorithm has the form $q+U(\mathcal{C}^{\circ})$ after puncturing, repeated incidence-twin compression, weight-1 preclusion, and layered superset sharpening on a residual family. Proposition~\ref{prop:punctured-certificate}, together with Theorem~\ref{thm:doc-dominance} and Proposition~\ref{prop:layered-superset-sharpening}, shows that each such quantity upper-bounds $\alpha(\mathcal{B})$. The output is the minimum over recorded valid certificates and is therefore valid. Termination follows because every accepted puncture deletes at least one live block, every weight-1 preclusion removes at least one nonempty block and its clustered vertices from the residual family, every nontrivial incidence-twin or superset-sharpening pass deletes at least one live vertex from the residual certificate family, and the initial incidence family is finite.
\end{proof}

\begin{example}[Two-round puncturing cascade]
\label{ex:puncturing-cascade}
Let the vertices be $x_1,x_2,x_3,x_4,y_1,y_2,y_3,y_4$, and define
\[
\begin{aligned}
A_1&=\{x_1,x_2\},&
A_2&=\{x_3,x_4\},&
A_3&=\{x_1,x_3\},&
A_4&=\{x_2,x_4\},\\
L_1&=\{x_1,y_1,y_2\},&
L_2&=\{x_3,y_3,y_4\},&
L_3&=\{x_1,x_3,y_1,y_3\},&
L_4&=\{x_2,x_4,y_2,y_4\}.
\end{aligned}
\]
The degrees are
\[
d(x_1)=d(x_3)=4,\qquad d(x_2)=d(x_4)=3,\qquad d(y_i)=2\quad(i=1,\dots,4),
\]
so
\[
w(A_3)=4,\qquad w(A_1)=w(A_2)=w(A_4)=3,\qquad w(L_i)=2\quad(i=1,\dots,4),
\]
and the initial closed-form certificate is
\[
U(\mathcal{B})=\frac{1}{4}+\frac{3}{3}+\frac{4}{2}=\frac{13}{4}.
\]

One admissible high-to-low puncturing trace first removes $A_3$, then $A_1$ and $A_2$, each of which decreases the current recomputed certificate. The dynamic increasing-$n_{\min}$ order is visible along the way: $A_3$ is the only block in layer $w=4$; after puncturing $A_3$, the layer $w=3$ contains $A_1,A_2,A_4$ with tied or updated minimum-degree counts, and after one of $A_1,A_2$ is punctured the other becomes the only remaining block in that highest layer.

The residual family is then
\[
A_4,\ L_1,\ L_2,\ L_3,\ L_4,
\]
with current certificate $5/2$. At this point all live blocks have weight $2$, but their current minimum-degree counts are
\[
n_{\min}(A_4)=2,\qquad n_{\min}(L_1)=n_{\min}(L_2)=3,\qquad
n_{\min}(L_3)=n_{\min}(L_4)=4.
\]
Thus the increasing-$n_{\min}$ rule selects $A_4$ before the larger $L$-blocks. Puncturing $A_4$ is a zero-gain move, but it creates weight-1 structure in the latent $L$-blocks. Algorithm~\ref{alg:greedy-puncturing} handles such newly exposed weight-1 blocks at the next outer rebuild, so the current layer can still continue its puncturing pass.

Continuing the same dynamically updated order, $L_1$ and $L_2$ are the smallest remaining weight-2 blocks by current $n_{\min}$, so they are punctured before the larger blocks $L_3,L_4$. This leaves only
\[
L_3=\{x_1,x_3,y_1,y_3\},\qquad
L_4=\{x_2,x_4,y_2,y_4\},
\]
both of residual weight $1$. The next rebuild therefore precludes $L_3$ and $L_4$, fixes two canonical units, leaves the residual family empty, and records
\[
U_{\mathrm{punct}}=2.
\]
The example isolates why Algorithm~\ref{alg:greedy-puncturing} is an outer-iteration procedure rather than a one-shot deletion rule: some useful punctures do not improve the certificate at the moment they are made, but the dynamic increasing-$n_{\min}$ order finds small exposure blocks first, these punctures expose the next wave of weight-1 blocks, and the certificate tightens only after the subsequent rebuild.
\end{example}

\subsection{Covering Certificates}

The block-weight bound, its weight-1 preclusion form, and the puncturing certificate provide instance-level upper bounds. They also identify low-$w(B)$ blocks as the most constraining parts of the instance, motivating both the covering certificate below and the greedy clustering schedule in the next section.

\begin{example}[A loose weight-2 block-weight bound]
\label{ex:pair-parity-family}
Let $x_1,\dots,x_{2n}$ be vertices with $n\ge 3$, and define pair blocks
\[
P_t:=\{x_{2t-1},x_{2t}\}\quad(t=1,\dots,n),
\]
together with the two parity blocks
\[
O:=\{x_1,x_3,\dots,x_{2n-1}\},\qquad
E:=\{x_2,x_4,\dots,x_{2n}\}.
\]
Then every vertex has degree $2$, and every block has minimum weight $w=2$. Hence Theorem~\ref{thm:block-weight-upper} gives
\[
\alpha(\mathcal{B})\le \frac{n+2}{2}.
\]
However, $\alpha(\mathcal{B})=2$. Indeed, one may retain two vertices, for example $\{x_1,x_4\}$. On the other hand, the parity blocks force any feasible set to contain at most one odd-indexed vertex and at most one even-indexed vertex. This family shows that the closed-form bound can remain loose even when all blocks have uniform weight~2 and no larger ``all-$x$'' block is present. Example~\ref{ex:pair-parity-trace} later traces this same family through puncturing, covering, and greedy clustering.
\end{example}

Theorem~\ref{thm:block-weight-upper} provides the foundational ordering principle for the two greedy procedures that follow. It shows that a block's contribution toward the maximum strong independent set size is governed primarily by the inverse of its minimum weight, so lower-weight blocks are the most constraining part of the instance and should be handled first.

The Greedy Layered Covering Algorithm below and the Greedy Layered Clustering Algorithm in Section~5 are complementary realizations of this same principle: both prioritize increasing block weight, and both refine ties by favoring blocks containing more vertices of that current minimum weight. The weight-2 example above then explains why a second, covering-based viewpoint is useful in addition to the closed-form charging bound of Theorem~\ref{thm:block-weight-upper}: concentrated overlap can force many nominal low-weight candidates into a small number of incompatibility blocks.

\begin{algorithm}[Greedy Layered Covering Algorithm and iteration]
\label{alg:greedy-covering}
Initialize a cumulative covered set $\mathcal{Z}_0:=\varnothing$, a residual block family
\[
\mathcal{F}^{(0)}:=\{B\in\mathcal{B}:B\neq\varnothing\},
\]
where residual blocks retain their original block indices, and a counter $N_{\mathrm{cov}}:=0$. Degrees are recomputed intrinsically on the current residual family:
\[
d_t(v):=\bigl|\{B\in\mathcal{F}^{(t)}:v\in B\}\bigr|.
\]
The \emph{Greedy Layered Covering Algorithm} repeatedly performs the following:
\begin{enumerate}[leftmargin=1.6em,label=\arabic*.,itemsep=0.15em]
    \item If $\mathcal{F}^{(t)}=\varnothing$, stop and return the selected covering family
    \[
    \mathcal{B}^\star:=\{B_0^\star,\dots,B_{N_{\mathrm{cov}}-1}^\star\}
    \]
    and the final counter $N_{\mathrm{cov}}=|\mathcal{B}^\star|$.
    \item Compute current block weights
    \[
    w_t(B):=\min_{v\in B} d_t(v),
    \]
    the current minimum layer
    \[
    \lambda_t:=\min_{B\in\mathcal{F}^{(t)}} w_t(B),
    \qquad
    \mathcal{M}_t:=\{B\in\mathcal{F}^{(t)}: w_t(B)=\lambda_t\},
    \]
    and, for each current degree-$\lambda_t$ vertex, the layer-local vertex-to-block index
    \[
    \mathcal{I}_{\lambda_t}(v):=\{\,i:\ B_i\in\mathcal{M}_t,\ v\in B_i,\ d_t(v)=\lambda_t\,\},
    \]
    together with, for each $B\in\mathcal{M}_t$, the number of current minimum-weight vertices
    \[
    n_t(B):=\bigl|\{v\in B:d_t(v)=\lambda_t\}\bigr|.
    \]
    \item Select a covering block
    \[
    B_t^\star\in\arg\max_{B\in\mathcal{M}_t} n_t(B),
    \]
    breaking ties arbitrarily.
    \item Update
    \[
    \mathcal{Z}_{t+1}:=\mathcal{Z}_t\cup B_t^\star,
    \qquad
    N_{\mathrm{cov}}\leftarrow N_{\mathrm{cov}}+1,
    \]
    and update the residual family by removing the cumulative covered set from every original block:
    \begin{equation}
    \mathcal{F}^{(t+1)}
    :=
    \{\,B\setminus \mathcal{Z}_{t+1} : B\in\mathcal{B},\ B\setminus \mathcal{Z}_{t+1}\neq\varnothing\,\}.
    \end{equation}
\end{enumerate}
The iterated covering version used in the numerical tables sets $\mathcal{B}^{(0)}:=\mathcal{B}$ and repeats this covering pass on the current family. At iteration $s+1$, it runs the pass on $\mathcal{B}^{(s)}$, records the returned family as $\mathcal{B}^{(s+1)}$, records
\[
N_{\mathrm{cov}}^{(s+1)}:=|\mathcal{B}^{(s+1)}|,
\qquad
U_{\mathrm{cover}}^{(s+1)}:=U(\mathcal{B}^{(s+1)}),
\]
and stops when $\mathcal{B}^{(s+1)}=\mathcal{B}^{(s)}$. Thus each reported $U_{\mathrm{cover}}$ is the final restricted block-weight certificate at convergence, while the associated iteration count is the first $s+1$ for which the covering family is unchanged.
\end{algorithm}

Algorithm~\ref{alg:greedy-covering} deliberately uses the opposite layer priority from Algorithm~\ref{alg:greedy-puncturing}. Puncturing starts from the highest residual block weights and, inside a layer, chooses the smallest current $n_{\min,\mathcal{C}}(B)$ because such blocks are the safest candidates to remove without worsening the certificate. Covering starts from the lowest residual block weight and chooses the largest current $n_t(B)$ because such blocks cover the most active minimum-degree vertices in the bottleneck layer.

\begin{remark}[Two-pass implementation within a fixed covering layer]
For implementation, Step~3 of Algorithm~\ref{alg:greedy-covering} need not recompute $n_t(B)$ over all residual blocks after every covering choice. Once a fixed layer $w=\lambda_t$ is reached, one may perform two passes, using $d_{\mathcal{F}}$ for the current residual-family degree. In the first pass, scan all current weight-$w$ blocks, compute and store their counts
\[
n(B):=\bigl|\{v\in B:d_{\mathcal{F}}(v)=w\}\bigr|,
\]
and build the layer-local degree-$w$ vertex-to-block index $\mathcal{I}_w(\cdot)$.

In the second pass, repeatedly choose an unprocessed block with maximum current $n(B)$, add all of its vertices to the covered set, and for each of its degree-$w$ vertices $v$, decrement by one the stored count $n(\cdot)$ of every still-unprocessed associated block in $\mathcal{I}_w(v)$. If the count of an unprocessed block drops to zero, then all of its degree-$w$ vertices have been covered, so its residual weight has increased beyond $w$; it may therefore be removed from the current layer and appended to a higher-layer work list, for example the next layer queue.

This two-pass realization uses temporary storage for $\mathcal{I}_w(\cdot)$, but it updates only the blocks actually touched by the newly covered degree-$w$ vertices and therefore avoids rescanning the entire residual layer after each covering selection.
\end{remark}

\begin{theorem}[Greedy covering upper bound]
\label{thm:greedy-covering-upper}
Algorithm~\ref{alg:greedy-covering} returns an upper bound on the maximum feasible retention value:
\begin{equation}
\alpha(\mathcal{B})\le N_{\mathrm{cov}}.
\end{equation}
\end{theorem}

\begin{proof}
Let the algorithm select blocks
\[
B_0^\star,B_1^\star,\dots,B_{N_{\mathrm{cov}}-1}^\star.
\]
At termination, the residual family is empty. Following the residual updates shows that every original block $B\in\mathcal{B}$ satisfies
\[
B\setminus \mathcal{Z}_{N_{\mathrm{cov}}}=\varnothing,
\]
hence
\[
B\subseteq \mathcal{Z}_{N_{\mathrm{cov}}}
=
\bigcup_{t=0}^{N_{\mathrm{cov}}-1} B_t^\star.
\]
Therefore every retained root belongs to at least one selected covering block. Assign each retained root to the earliest selected block containing it. This assignment is injective, because each selected block may contain at most one retained root by feasibility. Hence the number of retained roots is at most the number of selected covering blocks:
\[
\alpha(\mathcal{B})\le N_{\mathrm{cov}}.
\]
\end{proof}

The next lemma separates two roles of a restricted block family that preserves the active vertex universe: dropping blocks relaxes feasibility, while under a stronger pairwise-cover condition the optimum is preserved exactly. Throughout the lemma, let $\mathcal{B}^\star\subseteq\mathcal{B}$ satisfy
\[
\bigcup_{B\in\mathcal{B}^\star}B=\bigcup_{B\in\mathcal{B}}B.
\]

\begin{lemma}[Subfamily monotonicity and equality under pairwise covering]
\label{lem:subfamily-monotonicity}
Under this shared-universe assumption:
\begin{enumerate}[leftmargin=1.8em,itemsep=0.15em,label=(\roman*)]
    \item One always has
    \[
    \alpha(\mathcal{B})\le \alpha(\mathcal{B}^\star).
    \]
    \item If, in addition, for every omitted block $B\in\mathcal{B}\setminus\mathcal{B}^\star$ and every distinct $x,y\in B$, there exists $C\in\mathcal{B}^\star$ such that
    \[
    \{x,y\}\subseteq C,
    \]
    then
    \[
    \alpha(\mathcal{B})=\alpha(\mathcal{B}^\star).
    \]
\end{enumerate}
\end{lemma}

\begin{proof}
For part (i), the two families have the same active universe, and every $\mathcal{B}$-feasible set is automatically $\mathcal{B}^\star$-feasible because $\mathcal{B}^\star$ is a subfamily. Hence $\alpha(\mathcal{B})\le \alpha(\mathcal{B}^\star)$.

For part (ii), let $R$ be any $\mathcal{B}^\star$-feasible set. If some omitted block $B\in\mathcal{B}\setminus\mathcal{B}^\star$ contained two retained vertices $x,y\in R$, then by assumption there would exist $C\in\mathcal{B}^\star$ with $\{x,y\}\subseteq C$, contradicting feasibility of $R$ for $\mathcal{B}^\star$. Hence every $\mathcal{B}^\star$-feasible set is also $\mathcal{B}$-feasible, so $\alpha(\mathcal{B}^\star)\le \alpha(\mathcal{B})$. Together with part (i), this gives equality.
\end{proof}

\begin{remark}[Illustration of subfamily monotonicity]
Take
\[
\mathcal{B}
=
\Bigl\{
B_0=\{a,b,c\},\;
B_1=\{a,b,d\},\;
B_2=\{a,c,e\},\;
B_3=\{b,c,f\}
\Bigr\},
\]
and let
\[
\mathcal{B}^\star=\{B_1,B_2,B_3\}.
\]
The omitted block $B_0$ is not itself in $\mathcal{B}^\star$, but every pair in $B_0$ is covered inside $\mathcal{B}^\star$:
\[
\{a,b\}\subseteq B_1,\qquad
\{a,c\}\subseteq B_2,\qquad
\{b,c\}\subseteq B_3.
\]
Hence Lemma~\ref{lem:subfamily-monotonicity}(ii) applies and gives
\[
\alpha(\mathcal{B})=\alpha(\mathcal{B}^\star).
\]
In this example the common optimum is $3$, achieved for instance by the feasible set $\{d,e,f\}$. The point is that all pairwise conflicts inside $B_0$ are already enforced by the selected subfamily $\mathcal{B}^\star$.

By contrast, plain covering alone does not force equality. Consider
\[
\widetilde{\mathcal{B}}
=
\Bigl\{
\widetilde B_0=\{a,b,c\},\;
\widetilde B_1=\{a,b,d\},\;
\widetilde B_2=\{a,c,d\}
\Bigr\},
\qquad
\widetilde{\mathcal{B}}^\star=\{\widetilde B_1,\widetilde B_2\}.
\]
The family $\widetilde{\mathcal{B}}^\star$ is a covering subfamily because
\[
\widetilde B_0\subseteq \widetilde B_1\cup \widetilde B_2.
\]
However the pair $\{b,c\}$ is not contained in any block of $\widetilde{\mathcal{B}}^\star$, so the pairwise-cover condition of Lemma~\ref{lem:subfamily-monotonicity}(ii) fails. Indeed,
\[
\alpha(\widetilde{\mathcal{B}})=1
\qquad\text{whereas}\qquad
\alpha(\widetilde{\mathcal{B}}^\star)=2,
\]
because $\{b,c\}$ is feasible for $\widetilde{\mathcal{B}}^\star$ but not for $\widetilde{\mathcal{B}}$. This gives a strict instance of part~(i):
\[
\alpha(\widetilde{\mathcal{B}})<\alpha(\widetilde{\mathcal{B}}^\star).
\]
\end{remark}

The restricted covering theorem below uses the monotonicity part of Lemma~\ref{lem:subfamily-monotonicity}, together with the fact that a covering family preserves the active vertex universe: once such a family $\mathcal{B}^\star$ has been identified, Theorem~\ref{thm:block-weight-upper} can be reapplied on that restricted family to obtain a sharper upper bound. Here $\mathcal{B}^\star\subseteq\mathcal{B}$ is a covering subfamily, meaning that every block $B\in\mathcal{B}$ satisfies
\[
B\subseteq \bigcup_{C\in\mathcal{B}^\star} C.
\]
Equivalently, on the active block universe, every vertex that appears in the original block family still appears in the restricted family.

\begin{theorem}[Restricted block-weight refinement over any covering family]
\label{thm:covering-refinement}
If $\mathcal{B}^\star\subseteq\mathcal{B}$ is a covering subfamily, then
\begin{equation}
\alpha(\mathcal{B})\le U(\mathcal{B}^\star)\le |\mathcal{B}^\star|.
\end{equation}
In particular, taking $\mathcal{B}^\star=\{B_0^\star,\dots,B_{N_{\mathrm{cov}}-1}^\star\}$ from Theorem~\ref{thm:greedy-covering-upper} gives
\[
\alpha(\mathcal{B})\le U(\mathcal{B}^\star)\le N_{\mathrm{cov}}.
\]
\end{theorem}

\begin{proof}
By part~(i) of Lemma~\ref{lem:subfamily-monotonicity},
\[
\alpha(\mathcal{B})\le \alpha(\mathcal{B}^\star).
\]
Applying Theorem~\ref{thm:block-weight-upper} to the restricted family $\mathcal{B}^\star$ gives
\[
\alpha(\mathcal{B}^\star)\le U(\mathcal{B}^\star).
\]
Here the covering condition is what makes the restricted-family certificate legitimate for the same active vertex universe: no originally active vertex has disappeared from all retained blocks.
Combining the two inequalities yields
\[
\alpha(\mathcal{B})\le U(\mathcal{B}^\star).
\]
Finally, every block weight contributing to $U(\mathcal{B}^\star)$ is at least $1$,
\[
U(\mathcal{B}^\star)\le |\mathcal{B}^\star|.
\]
\end{proof}

The outer iteration in Algorithm~\ref{alg:greedy-covering} is justified by the same covering-and-reweighting principle: once one covering family has been extracted, the pass may be applied again to that restricted family, yielding a nested sequence of valid upper bounds whose limit family has a simple fixed-point structure. Initialize
\[
\mathcal{B}^{(0)}:=\mathcal{B},
\]
and for each $t\ge 0$, let $\mathcal{B}^{(t+1)}\subseteq \mathcal{B}^{(t)}$ be the covering family produced by running the Greedy Layered Covering Algorithm on $\mathcal{B}^{(t)}$. If $\mathcal{B}^{(t+1)}=\mathcal{B}^{(t)}$, the iteration has converged and stops; otherwise $\mathcal{B}^{(t+1)}\subset \mathcal{B}^{(t)}$. Let $w_t^\star(B)$ denote the restricted block weight computed inside $\mathcal{B}^{(t)}$.

\begin{theorem}[Iterated covering refinement and fixed-point structure]
\label{thm:cover-fixed-point}
For the covering sequence above, every constructed $t\ge 1$ satisfies
\begin{equation}
\alpha(\mathcal{B})\le U(\mathcal{B}^{(t)})\le |\mathcal{B}^{(t)}|,
\end{equation}
and the sequence $|\mathcal{B}^{(t)}|$ is nonincreasing and stabilizes after finitely many steps. Therefore there exists some $t_{\mathrm{cov}}\ge 0$ such that
\[
\mathcal{B}^{(t_{\mathrm{cov}}+1)}=\mathcal{B}^{(t_{\mathrm{cov}})},
\]
and every block in the converged family has restricted weight $1$:
\[
w_{t_{\mathrm{cov}}}^\star(B)=1
\qquad
\forall B\in\mathcal{B}^{(t_{\mathrm{cov}})},
\]
and consequently
\[
U(\mathcal{B}^{(t_{\mathrm{cov}})})
=
|\mathcal{B}^{(t_{\mathrm{cov}})}|.
\]
\end{theorem}

\begin{proof}
We first prove the validity and termination claims.
By construction, each $\mathcal{B}^{(t+1)}$ is a covering subfamily of $\mathcal{B}^{(t)}$. Applying Theorem~\ref{thm:covering-refinement} with $\mathcal{B}=\mathcal{B}^{(t-1)}$ and $\mathcal{B}^\star=\mathcal{B}^{(t)}$ yields
\[
\alpha(\mathcal{B}^{(t-1)})\le U(\mathcal{B}^{(t)})\le |\mathcal{B}^{(t)}|.
\]
Since $\mathcal{B}^{(t)}\subseteq \mathcal{B}^{(t-1)}$ for every $t$, Lemma~\ref{lem:subfamily-monotonicity}(i) gives
\[
\alpha(\mathcal{B})=\alpha(\mathcal{B}^{(0)})\le \alpha(\mathcal{B}^{(t-1)}),
\]
and therefore
\[
\alpha(\mathcal{B})\le U(\mathcal{B}^{(t)})\le |\mathcal{B}^{(t)}|.
\]
Finally, because $\mathcal{B}^{(t+1)}\subseteq \mathcal{B}^{(t)}$ and the family sizes are nonnegative integers, the sequence $|\mathcal{B}^{(t)}|$ is nonincreasing and must stabilize after finitely many iterations.

It remains to identify the structure of a converged family.
Now let $t_{\mathrm{cov}}$ be a converged index. Run the Greedy Layered Covering Algorithm on the fixed family $\mathcal{B}^{(t_{\mathrm{cov}})}$. Since the covering family returned by that run is again $\mathcal{B}^{(t_{\mathrm{cov}})}$, every block of $\mathcal{B}^{(t_{\mathrm{cov}})}$ is eventually selected during the run.

Assume for contradiction that some block of $\mathcal{B}^{(t_{\mathrm{cov}})}$ has restricted weight greater than $1$, and let
\[
\widehat{\mathcal B}_{>1}
:=
\{\,B\in\mathcal{B}^{(t_{\mathrm{cov}})}: w_{t_{\mathrm{cov}}}^\star(B)>1\,\}.
\]
Choose $B_i\in \widehat{\mathcal B}_{>1}$ to be the \emph{last block actually selected} among all blocks of $\widehat{\mathcal B}_{>1}$ during this fixed-point rerun.

Take any vertex $v\in B_i$. Since $w_{t_{\mathrm{cov}}}^\star(B_i)>1$, every vertex of $B_i$ has restricted degree at least $2$ inside $\mathcal{B}^{(t_{\mathrm{cov}})}$, so there exists another block $C_v\neq B_i$ in $\mathcal{B}^{(t_{\mathrm{cov}})}$ containing $v$.

If $w_{t_{\mathrm{cov}}}^\star(C_v)=1$, then $C_v$ is selected earlier than $B_i$ because all weight-$1$ blocks are processed before any block of intrinsic weight greater than $1$. If $w_{t_{\mathrm{cov}}}^\star(C_v)>1$, then $C_v\in \widehat{\mathcal B}_{>1}$, and by the choice of $B_i$ as the last selected block in $\widehat{\mathcal B}_{>1}$, the block $C_v$ is again selected earlier than $B_i$.

In either case, the earlier selection of $C_v$ covers the vertex $v$ before $B_i$ would be selected. Since this argument applies to every vertex $v\in B_i$, all elements of $B_i$ are already covered before its supposed selection time. Therefore the residual version of $B_i$ is empty at that moment, so $B_i$ must be discarded rather than selected. This contradicts the fact that every block of the fixed-point family $\mathcal{B}^{(t_{\mathrm{cov}})}$ is selected again in the rerun.

Therefore no converged block can have restricted weight greater than $1$, proving that $w_{t_{\mathrm{cov}}}^\star(B)=1$ for every $B\in\mathcal{B}^{(t_{\mathrm{cov}})}$. The displayed identity follows immediately.
\end{proof}

Apply the iterated covering procedure until convergence, and denote its terminal covering family by
\[
\widehat{\mathcal{B}}:=\mathcal{B}^{(t_{\mathrm{cov}})}.
\]
Since Theorem~\ref{thm:cover-fixed-point} gives unit block weight in $\widehat{\mathcal{B}}$, choose for each terminal block $B\in\widehat{\mathcal{B}}$ one vertex $z_B\in B$ with
\[
d_{\widehat{\mathcal{B}}}(z_B)=1.
\]
Set
\[
Z:=\{z_B:B\in\widehat{\mathcal{B}}\}.
\]
The next corollary records the canonical structure of this terminal restricted family.

\begin{corollary}[Unit-weight covering family and canonical optimum]
For every finite block family $\mathcal{B}$, the terminal family $\widehat{\mathcal{B}}$ is a covering subfamily of $\mathcal{B}$ satisfying
\[
w_{\widehat{\mathcal{B}}}(B)=1
\qquad
\forall B\in\widehat{\mathcal{B}}.
\]
Moreover, $Z$ is an optimal retained set for $\widehat{\mathcal{B}}$ satisfying
\[
|Z|
=|\widehat{\mathcal{B}}|=\alpha(\widehat{\mathcal{B}})=U(\widehat{\mathcal{B}}).
\]
\end{corollary}

\begin{proof}
Theorem~\ref{thm:cover-fixed-point} shows that $\widehat{\mathcal{B}}$ is a covering subfamily of $\mathcal{B}$ and that every block in $\widehat{\mathcal{B}}$ has intrinsic weight $1$. Therefore Theorem~\ref{thm:w1-canonical}, applied to the restricted family $\widehat{\mathcal{B}}$, gives an optimal feasible set containing one restricted-degree-$1$ vertex $z_B$ from every block $B\in\widehat{\mathcal{B}}$. Since each such vertex belongs to exactly one block of $\widehat{\mathcal{B}}$, the selected set $Z$ is feasible and satisfies $|Z|=|\widehat{\mathcal{B}}|$. The same fixed-point theorem also gives $U(\widehat{\mathcal{B}})=|\widehat{\mathcal{B}}|$, and Theorem~\ref{thm:block-weight-upper} applied to $\widehat{\mathcal{B}}$ yields $\alpha(\widehat{\mathcal{B}})\le U(\widehat{\mathcal{B}})$. Hence
\[
|Z|=|\widehat{\mathcal{B}}|=U(\widehat{\mathcal{B}})\ge \alpha(\widehat{\mathcal{B}}).
\]
Since $Z$ is itself feasible for $\widehat{\mathcal{B}}$, one also has $|Z|\le \alpha(\widehat{\mathcal{B}})$, proving the claimed equalities.
\end{proof}

For later comparison, it is convenient to name the main structural bounds introduced in this section:
\begin{equation}
U_{\mathrm{cf}}
:=
U(\mathcal{B}),
\qquad
U_{\mathrm{w1}}
:=
|\mathcal{B}_1|+U(\mathcal{B}^\circ),
\qquad
U_{\mathrm{punct}}
:=\operatorname{Punct}(\mathcal{B}),
\end{equation}
for the closed-form, weight-1 preclusion, and iterative puncturing bounds, respectively, where $\operatorname{Punct}(\mathcal{B})$ denotes the certificate returned by Algorithm~\ref{alg:greedy-puncturing}. For an iterated covering sequence $\mathcal{B}^{(t)}$, we also write
\begin{equation}
N_{\mathrm{cov}}^{(t)}
:=
|\mathcal{B}^{(t)}|,
\qquad
U_{\mathrm{cover}}^{(t)}
:=
U(\mathcal{B}^{(t)}).
\end{equation}
In particular, the first covering round satisfies $N_{\mathrm{cov}}^{(1)}=|\mathcal{B}^\star|$ and $U_{\mathrm{cover}}^{(1)}=U(\mathcal{B}^\star)$. When the iterated sequence is run to convergence in the empirical tables, $U_{\mathrm{cover}}$ denotes the final value of $U_{\mathrm{cover}}^{(t)}$.

The puncturing and iterated-covering refinements are complementary rather than inclusive. The iterative puncturing bound $U_{\mathrm{punct}}$ is obtained by searching over residual block subfamilies across outer rounds, recomputing intrinsic weights after punctures, and repeatedly precluding newly exposed weight-1 blocks; the iterated covering bounds $U_{\mathrm{cover}}^{(t)}$ instead come from repeatedly selecting covering families and recomputing intrinsic weights on those restricted families. Because these mechanisms are different, neither quantity dominates the other in general. In practice, one should therefore compute both and use
\[
\min\bigl\{U_{\mathrm{punct}},\, U_{\mathrm{cover}}^{(1)},\,U_{\mathrm{cover}}^{(2)},\,\dots\bigr\}
\]
as the strongest available structural comparison target.

\begin{runningexamplepart}[Running example under the Greedy Layered Covering Algorithm]
\label{ex:running-covering}
The running example from Examples~\ref{ex:running-example-family}--\ref{ex:running-example-bound} also behaves cleanly under the Greedy Layered Covering Algorithm. One valid covering run selects
\[
\mathcal{B}^\star=\{B_1,B_2,B_4,B_8,B_{11}\},
\]
so Theorem~\ref{thm:greedy-covering-upper} gives
\[
\alpha(\mathcal{B})\le N_{\mathrm{cov}}=5.
\]
Moreover, in the restricted family $\mathcal{B}^\star$, each selected block has restricted weight $1$: for example, $B_1$ contains the restricted-degree-1 vertex $a$, $B_2$ contains $b$, and $B_4,B_8,B_{11}$ contain $e,c,r$, respectively. Hence
\[
U(\mathcal{B}^\star)
=
1+1+1+1+1
=
5,
\]
so Theorem~\ref{thm:covering-refinement} reaches the same exact value. In this instance the dominance-reduced closed-form bound, the covering-family restricted bound, and the closed-form bound for the selected covering set all coincide with the true optimum.
\end{runningexamplepart}

This algorithm is complementary to Theorem~\ref{thm:block-weight-upper}: Theorem~\ref{thm:block-weight-upper} gives a closed-form charging bound, while the Greedy Layered Covering Algorithm searches directly for a small covering family of incompatibility blocks in the active minimum-weight layer. Theorem~\ref{thm:covering-refinement} then shows that any such covering family, regardless of how it is obtained, yields a restricted block-weight upper bound.

The practical purpose of the greedy construction is to capture concentration effects like the weight-2 example above, where many nominal low-weight candidates collapse into a single forced block. The secondary key $n_t(B)$ is not needed for validity, but it is the natural greedy choice because it attempts to cover as many current degree-$\lambda_t$ candidates as possible per unit increase of the bound. The tradeoff is that, unlike Theorem~\ref{thm:block-weight-upper}, this procedure has no closed form and requires repeated residual updates, so its direct computation is typically quadratic-scale in the residual incidences.

\begin{remark}[Closed-form, covering-count, and restricted-covering bounds are complementary]
Consider the block family
\[
\{a_1,a_2,a_3,a_4,a_6\},\;
\{a_1,a_2,a_5,a_6,a_7\},\;
\{a_3,a_4,a_5,a_7,a_8\},\;
\{a_2,a_4,a_7\},\;
\{a_1,a_3,a_8\},\;
\{a_5,a_6,a_8\}.
\]
Here each vertex has degree $3$, so all blocks have weight $3$ and Theorem~\ref{thm:block-weight-upper} gives $6/3=2$. By contrast, the Greedy Layered Covering Algorithm of Theorem~\ref{thm:greedy-covering-upper} may select the three covering blocks
\[
\mathcal{B}^\star=\{B_0,B_1,B_2\},
\]
yielding the coarser count bound $N_{\mathrm{cov}}=3$, even though one minimal covering family is already $\{B_1,B_2\}$. However, Theorem~\ref{thm:covering-refinement} recovers the sharper value from the selected family itself: within $\mathcal{B}^\star$ one has
\[
w^\star(B_0)=2,\qquad w^\star(B_1)=2,\qquad w^\star(B_2)=1,
\]
so
\[
U(\mathcal{B}^\star)
\;=\;
\frac12+\frac12+1
\;=\;2.
\]
This matches the exact optimum. Indeed, the unique optimal retained set is
\[
\{a_2,a_8\},
\]
because $a_2$ and $a_8$ are the only pair of vertices that never co-occur in a block, whereas every other vertex conflicts with all remaining candidates. Thus this example shows all three viewpoints are genuinely different: the closed-form theorem can already be sharp on balanced instances, the greedy covering count from Theorem~\ref{thm:greedy-covering-upper} may overcount through a non-minimal covering choice, and Theorem~\ref{thm:covering-refinement} can tighten that covering-based bound back down by reweighting the selected covering family.

Theorem~\ref{thm:cover-fixed-point} makes the picture sharper in two ways. First, for the weight-2 family above, a covering family such as $\{O,E,P_1\}$ gives the restricted-weight bound $1+1+1=3$: this is still not exact, but it collapses the gap from asymptotically growing with $n$ to a constant-size gap. Second, for the present example, a second covering iteration reduces the family to $\{B_1,B_2\}$, yielding the tight bound
\[
1+1
=
2.
\]
This favorable outcome also depends on the secondary greedy order inside each weight layer: the decreasing-$n_{\min}(B)$ rule is what drives the covering routine toward the stronger families used here, whereas a weaker within-layer order need not produce the same refinement.
\end{remark}

\section{Greedy Layered block-Weight Clustering}

The previous section develops structural bounds for the maximum strong independent set objective on a hyperedge family. We now turn those bounds into an explicit block-native solver. In particular, Theorems~\ref{thm:block-weight-upper} and~\ref{thm:greedy-covering-upper} provide the first-order and second-order scheduling principles that drive the block-weight-layered greedy clustering routine below.

Generic bounded-degree hypergraph algorithms typically treat a hypergraph as a graph-style combinatorial object in a more general setting \cite{halldorsson2009}. In contrast, the approach here is explicitly incidence-native: it assumes the block family itself is materialized and uses block weights, minimum-degree witnesses, and block-side residual updates as the primary state.

Theorem~\ref{thm:block-weight-upper} provides the first-order principle: blocks of smaller minimum weight are more constraining and therefore should be processed first. Theorem~\ref{thm:greedy-covering-upper} provides a second-order principle: within a fixed block-weight layer, blocks containing more minimum-weight vertices exert stronger covering pressure and should be prioritized next. This suggests a block-native greedy algorithm organized by block-weight layers with a lexicographic within-layer priority, rather than by an auxiliary hypergraph optimization routine.

Let $\mathcal{B}=\{B_1,\dots,B_M\}$ denote the incidence-twin-compressed block family. The algorithm below is stated on an abstract indexed family, but its operations are deliberately block-local: residual deletion and layer updates are expressed as scans of block records and vertex-to-block incidence lists.

Incidence-twin compression quotients vertices with identical block-index support and stores their original vertex lists for final materialization if needed. All initial degrees, block weights, and $n_{\min}$ values are computed on this quotient family; this avoids treating mutually exclusive identical-support vertices as separate minimum-weight candidates.

Following Corollary~\ref{cor:w1-elim-upper}, the clustering routine first performs a weight-1 preclusion pass. For each $B\in\mathcal{B}_1:=\{B\in\mathcal{B}:w(B)=1\}$, it selects a degree-1 vertex class $z_B$ as canonical root, attaches currently unassigned members of $B$ to that root, and marks attached classes as removed from subsequent processing. It then removes attached classes from all remaining blocks and keeps the resulting nonempty residual family $\mathcal{B}^{\mathrm{res}}$. Degrees are recomputed on $\mathcal{B}^{\mathrm{res}}$ and kept fixed during the later layered pass; once a class is attached, it is removed from subsequent layer processing rather than assigned a new residual degree.

In implementation, we materialize $\mathcal{B}^{\mathrm{res}}$ as a mutable list-of-lists representation
\[
\mathbf{B}=[B_1,\dots,B_{M_{\mathrm{res}}}],
\]
where each $B_i$ is updated in place during processing. Layer priorities use
\[
\widehat w(B):=\min_{v\in B}d(v),\qquad
n_{\min}(B):=\bigl|\{v\in B:d(v)=\widehat w(B)\}\bigr|,
\]
and for each integer weight $w$ we maintain a block-index set
\[
\mathcal{S}_w:=\{i:\widehat w(B_i)=w\}.
\]
Within a fixed block-weight layer $\widehat w(B)=w$, Theorem~\ref{thm:greedy-covering-upper} suggests prioritizing larger $n_{\min}(B)$: such blocks cover more current degree-$w$ candidates and therefore exert stronger covering pressure. Representative state is maintained by a root set $R$ and a vertex-to-root map $\phi:V\to V$ (vertex $\mapsto$ current root), initialized from the weight-1 preclusion pass and then updated by the layered greedy stage. The blocks are released progressively ($B_i\leftarrow\varnothing$) until all residual blocks are consumed.

\begin{definition}[Active root]
A vertex $u$ is an active root if it is currently chosen as the representative of its cluster.
\end{definition}

\begin{definition}[vertex-to-root map]
A vertex-to-root map is a function $\phi:V\to V$ such that $\phi(v)$ is the current representative of the cluster containing $v$, and $\phi(v)=v$ iff $v$ is a root.
\end{definition}

\begin{definition}[Cluster]
For a root $u$, define
\begin{equation}
\mathcal{C}_u := \{v\in V : \phi(v)=u\}.
\end{equation}
\end{definition}

\begin{definition}[Root selection key]
After a block $B$ has been chosen by the block-level priority, let
\begin{equation}
w(B):=\min_{v\in B} d(v),\qquad
M(B):=\{v\in B:d(v)=w(B)\}.
\end{equation}
Thus
\[
n_{\min}(B)=|M(B)|.
\]
For a candidate $v\in M(B)$, define its equal-weight collision neighborhood
\begin{equation}
\mathcal{N}_{w(B)}(v):=\left\{u\in \bigcup_{B'\ni v} B' : d(u)=w(B)\right\},
\end{equation}
and its equal-weight collision score
\begin{equation}
c_{w(B)}(v):=\left|\mathcal{N}_{w(B)}(v)\right|.
\end{equation}
Since each $v\in M(B)$ lies in exactly $w(B)$ blocks, $c_{w(B)}(v)$ counts the number of distinct degree-$w(B)$ vertices coupled to $v$ through those blocks. Thus, after the block $B$ has been selected, the root is chosen from $M(B)$ by
\begin{equation}
v^\star=\arg\min_{v\in M(B)}\bigl(c_{w(B)}(v),\mathrm{id}(v)\bigr).
\end{equation}
This can also be extended to include source-priority terms without changing feasibility.
\end{definition}

Fix a block $B$ and two candidate roots $u,v\in M(B)$, so that
\[
d(u)=d(v)=w(B).
\]

\begin{proposition}[Equal-weight collision minimization]
The number of current degree-$w(B)$ candidates consumed immediately by the root choice is monotone in the collision score:
\[
c_{w(B)}(u)\le c_{w(B)}(v)
\quad\Longrightarrow\quad
|\mathcal{N}_{w(B)}(u)|\le |\mathcal{N}_{w(B)}(v)|.
\]
Consequently, after a block has been selected, choosing from $M(B)$ the candidate with the smallest $c_{w(B)}(\cdot)$ minimizes the number of degree-$w(B)$ candidates clustered away at that step, and therefore preserves the largest remaining pool of degree-$w(B)$ candidates for future roots in the same layer.
\end{proposition}

\begin{proof}
By definition, $\mathcal{N}_{w(B)}(v)$ is exactly the set of current degree-$w(B)$ vertices coupled to $v$ through the blocks incident to $v$. Once $v$ is chosen as root, these coupled degree-$w(B)$ vertices are clustered away from future root consideration in the current layer. Hence the number of degree-$w(B)$ candidates consumed immediately by choosing $v$ is exactly
\[
c_{w(B)}(v)=|\mathcal{N}_{w(B)}(v)|.
\]
Therefore minimizing the collision score minimizes the number of current degree-$w(B)$ candidates removed at that step, leaving the largest remaining pool available for subsequent root creation in the same layer.
\end{proof}

Note that $c_{w(B)}(v)$ counts distinct degree-$w(B)$ competitors, so in general
\[
c_{w(B)}(v)\le \sum_{B'\ni v} n_{\min}(B'),
\]
with equality only when those minimum-weight candidate sets are disjoint outside of $v$ itself.

\begin{algorithm}[Greedy Layered Clustering Algorithm]
\label{alg:greedy-clustering}
The clustering stage proceeds as follows.
\begin{enumerate}[leftmargin=*,itemsep=0.2em]
    \item Load the indexed blocks, compress incidence-twin vertex classes, and denote the resulting quotient family by $\mathcal{B}=\{B_1,\dots,B_M\}$. Count quotient-vertex frequencies $d(v)$.
    \item Weight-1 preclusion: identify $\mathcal{B}_1:=\{B\in\mathcal{B}:w(B)=1\}$. For each $B\in\mathcal{B}_1$, pick a degree-1 root $z_B$, attach currently unassigned members of $B$ to $z_B$, and update $R,\phi$.
    \item Remove already clustered vertices from remaining blocks, set $\mathcal{B}^{\mathrm{res}}$ to the resulting nonempty residual block family, materialize it as $\mathbf{B}=[B_1,\dots,B_{M_{\mathrm{res}}}]$, recompute residual frequencies $d(v)$ on $\mathcal{B}^{\mathrm{res}}$, and for each residual block index $i$ compute $\widehat w(B_i)=\min_{v\in B_i}d(v)$ and insert $i$ into the corresponding block-index set $\mathcal{S}_{\widehat w(B_i)}$.
    \item For each block-weight layer $w$ from the current minimum upward:
    \begin{enumerate}[leftmargin=1.8em,itemsep=0.15em,label=(\alph*)]
        \item First pass over $\mathcal{S}_w$: for each $i\in\mathcal{S}_w$, remove already clustered vertices from $B_i$. If $B_i=\varnothing$, discard it. If $B_i$ contains no degree-$w$ vertex, recompute $\widehat w(B_i)>w$ and move $i$ to the corresponding larger-weight set $\mathcal{S}_{\widehat w(B_i)}$. Otherwise compute $n_{\min}(B_i)=|\{v\in B_i:d(v)=w\}|$ and keep $i$ in the active layer.
        \item From the surviving current-layer blocks, build the degree-$w$ vertex-to-block index
        \[
        \mathcal{I}_w(v):=\{i\in\mathcal{S}_w:v\in B_i,\ d(v)=w\},
        \]
        for each active degree-$w$ vertex.
        \item While active blocks remain in the current layer and the index $\mathcal{I}_w(\cdot)$ is nonempty, choose an active block $B_i$ with maximum current $n_{\min}(B_i)$, breaking ties arbitrarily.
        \item For the chosen block, define $M(B_i):=\{v\in B_i:d(v)=w\}$. If $M(B_i)=\varnothing$, recompute $\widehat w(B_i)>w$, move $i$ to the corresponding larger-weight set $\mathcal{S}_{\widehat w(B_i)}$, and continue. Otherwise compute the current collision scores $c_w(v)$ for $v\in M(B_i)$ and choose
        \[
        u^\star=\arg\min_{v\in M(B_i)}\bigl(c_w(v),\mathrm{id}(v)\bigr).
        \]
        \item Let $\mathcal{I}_w(u^\star)$ be the associated degree-$w$ block indices of the new root. Cluster all currently unassigned vertices in
        \[
        \bigcup_{j\in \mathcal{I}_w(u^\star)} B_j
        \]
        to $u^\star$ and set those blocks to $\varnothing$. For each degree-$w$ vertex $v$ in that union, use the already-built index $\mathcal{I}_w(v)$ to decrement by one the current counts $n_{\min}(\cdot)$ of all still-active associated blocks, then delete $v$ from $\mathcal{I}_w(\cdot)$. If the count of an active block drops to $0$, then no degree-$w$ vertex remains in it; recompute $\widehat w(B_i)>w$ and move it to the corresponding larger-weight set $\mathcal{S}_{\widehat w(B_i)}$.
        \item Delete the completed layer set $\mathcal{S}_w$ and continue with the next nonempty weight layer.
    \end{enumerate}
    \item Output root set $R$ and vertex-to-root map $\phi$ (including both weight-1 preclusion and layered greedy assignments).
\end{enumerate}
\end{algorithm}

For one update step of Algorithm~\ref{alg:greedy-covering} or Algorithm~\ref{alg:greedy-clustering} on a current residual family $\mathcal{F}$ and block-weight layer $w$, let $\mathcal{Q}$ be the block collection absorbed by that step:
\[
\mathcal{Q}=\{B_t^\star\}
\]
for the covering algorithm, and
\[
\mathcal{Q}=\{B_j:j\in\mathcal{I}_w(u^\star)\}
\]
for the clustering algorithm. Let
\[
A(\mathcal{Q}):=\bigcup_{B\in\mathcal{Q}}B
\]
be the vertices covered or clustered by the step.

\begin{proposition}[Implicit suppression for block-weight witnesses]
\label{prop:implicit-superset-suppression}
If $a\in A(\mathcal{Q})$ is a current block-weight witness,
\[
d_{\mathcal{F}}(a)=w,
\]
and an active vertex $b$ satisfies
\[
I_{\mathcal{F}}(a)\subseteq I_{\mathcal{F}}(b),
\]
then $b\in A(\mathcal{Q})$ as well. Consequently, a vertex whose current support is a superset of an absorbed block-weight witness cannot survive beyond that absorbing step. Thus the low-to-high flow of Algorithms~\ref{alg:greedy-covering} and~\ref{alg:greedy-clustering} performs a local, implicit form of superset suppression for current minimum-degree witnesses, without running a global support-containment pass.
\end{proposition}

\begin{proof}
Since $a\in A(\mathcal{Q})$, there exists some absorbed block $B\in\mathcal{Q}$ with $a\in B$. Hence the index of $B$ belongs to $I_{\mathcal{F}}(a)$. The containment $I_{\mathcal{F}}(a)\subseteq I_{\mathcal{F}}(b)$ implies $b\in B$. Therefore $b\in A(\mathcal{Q})$.

In Algorithm~\ref{alg:greedy-covering}, all vertices in $A(\mathcal{Q})=B_t^\star$ are added to the covered set and removed from subsequent residual blocks. In Algorithm~\ref{alg:greedy-clustering}, all currently unassigned vertices in $A(\mathcal{Q})$ are assigned to the chosen root and removed from later layer processing. Thus any active support-superset $b$ is suppressed no later than the step that absorbs the current block-weight witness $a$.
\end{proof}

This proposition is deliberately weaker than the fixed-point statement of Algorithm~\ref{alg:layered-superset-sharpening}. It does not claim that covering or clustering computes a full dominance closure, nor does it claim a suppression property for arbitrary nonminimum vertices incidentally absorbed by a selected block. Rather, it explains why explicit superset reduction is unnecessary inside their low-to-high inner loops for the vertices that actually drive the layer order: once a selected block absorbs a current block-weight witness, every active vertex whose support contains that witness support is absorbed automatically.

This restriction to block-weight witnesses is also the right structural focus. The closed-form bound charges each block by the inverse of its minimum incident degree, so the vertices attaining, or nearly attaining, block weights determine the bottleneck terms in the certificate and drive both the covering score $n_t(B)$ and the clustering score $n_{\min}(B)$. Superset effects involving only higher-degree nonminimum vertices may exist, but they do not control the first-order block-weight certificate. The low-to-high algorithms therefore spend their local update budget on the near-minimum witnesses that dominate the bound, rather than on a global dominance closure over all vertices.

This algorithm is genuinely layered: each weight level is cleaned first, then resolved by repeated maximum-$n_{\min}(B)$ choices driven by the current degree-$w$ vertex-to-block index. It therefore keeps the two structural priorities suggested by Theorem~\ref{thm:block-weight-upper} and Theorem~\ref{thm:greedy-covering-upper} while localizing all updates to the blocks touched by the newly attached degree-$w$ vertices.

\begin{remark}[Efficient within-layer maintenance]
Algorithm~\ref{alg:greedy-clustering} reuses the same degree-$w$ vertex-to-block index $\mathcal{I}_w(\cdot)$ both to select the blocks incident to a new root and to update the affected $n_{\min}(\cdot)$ counters after selection. Consequently, a fixed layer can be processed without rescanning all surviving blocks after each root choice: each degree-$w$ vertex is deleted from $\mathcal{I}_w(\cdot)$ at most once, and each current-layer incidence is touched only when its degree-$w$ endpoint is attached away. With a bucketed queue or equivalent priority structure keyed by the current $n_{\min}(B)$ values, the within-layer implementation is therefore essentially linear in the active layer incidences.
\end{remark}

\begin{remark}[Score-indexed next-block retrieval]
In an implementation of Step~4(c), the next active block with largest current $n_{\min}(B)$ can be maintained explicitly rather than found by rescanning the whole layer. Besides the degree-$w$ vertex-to-block mapping table $\mathcal{I}_w(\cdot)$, one keeps a current score table with one entry per active layer block, together with the inverse table that stores the active block indices at each score level. Then the next block is obtained from the largest nonempty score class, and each selection step updates only the blocks incident to the removed degree-$w$ vertices by decrementing their current score and moving their index between score classes. Thus maximum-$n_{\min}(B)$ maintenance remains local to the touched incidences and avoids repeated full-layer scans.
\end{remark}

The following worked example continues the running example from Examples~\ref{ex:running-example-family}--\ref{ex:running-example-bound}. Here the focus is no longer on the incidence data or the closed-form sum, but on how the layered greedy procedure realizes that bound and why it naturally leads to the exactness theorem.

\begin{runningexamplepart}[Worked example for the greedy layered clustering algorithm]
\label{ex:running-greedy-clustering}
For the running example, now partition the block indices into
\[
\mathcal{J}_1=\{1\},\qquad
\mathcal{J}_2=\{2,3\},\qquad
\mathcal{J}_3=\{4,5\},\qquad
\mathcal{J}_4=\{6,7,8\},\qquad
\mathcal{J}_5=\{9,10,11,12\},
\]
with canonical quotient representatives
\[
z_1=a,\qquad z_2=b,\qquad z_3=e,\qquad z_4=c,\qquad z_5=r.
\]
Component contraction would still select only one vertex on this connected family, while the dominance-reduced block certificate gives the optimal value $5$. The remaining question is whether the greedy layered clustering algorithm can realize that value without an explicit global dominance pass. The key mechanism is support-superset suppression. The vertex $u$ satisfies
\[
I(u)=\{1,2\},
\]
so it strictly contains $I(a)=\{1\}$. The smaller-weight branch applies because its extra block is attached to the lower-weight block $B_1$:
\[
w(B_1)=1<2.
\]
The gate vertices $f$ and $g$ similarly satisfy
\[
I(f)=\{2,3,9\},
\qquad
I(g)=\{4,5,10\},
\]
with
\[
I(b)=\{2,3\}\subsetneq I(f),
\qquad
I(e)=\{4,5\}\subsetneq I(g).
\]
Thus $f$ and $g$ are suppressed by the weight-$2$ representatives $b$ and $e$, respectively. Before they are suppressed, they gate the blocks $B_9$ and $B_{10}$ down to apparent weight $3$; after suppression, those blocks recover their true reduced weight $4$.

The greedy layered clustering algorithm now proceeds as follows.
\begin{enumerate}[leftmargin=1.8em,itemsep=0.15em,label=(\roman*)]
    \item Weight-1 preclusion chooses $a$ from $B_1$, clustering away the support-superset distractor $u$. This is the algorithmic version of the first dominance deletion used by Theorem~\ref{thm:incidence-block-optimality}: $B_2$ loses $u$ before the weight-$2$ layer is processed.
    \item In the weight-2 layer, blocks $B_2,B_3$ have $b$ as their degree-2 representative, while $B_4,B_5$ have $e$ as their degree-2 representative. The algorithm chooses $b$ and clusters the union $B_2\cup B_3$, thereby removing $f$ and the helpers $p,q$ from those two blocks; it then chooses $e$ from $B_4,B_5$, removing $g$ from the remaining active family. This exposes $B_9$ and $B_{10}$ as true weight-$4$ blocks in the dominance-reduced view.
    \item In the weight-3 layer, the residual blocks are
    \[
    \begin{aligned}
    B_6=\{c,x\},\qquad B_7=\{c,y\},\qquad B_8=\{c,x,y\},\\
    B_9=\{r,x\},\qquad B_{10}=\{r,y\},\qquad
    B_{11}=\{r,x\},\qquad B_{12}=\{r,y\}.
    \end{aligned}
    \]
    The only degree-$3$ representative in this layer is $c$, so the algorithm chooses $c$ and nullifies $B_6,B_7,B_8$, thereby removing $x$ and $y$. The remaining weight-$4$ block then reduces to blocks containing only $r$, so the algorithm chooses $r$.
\end{enumerate}
The final root set is
\[
R_G=\{a,b,e,c,r\},
\]
so
\[
|R_G|=5.
\]
Thus the greedy layered clustering algorithm exactly attains the dominance-reduced closed-form bound, while connected-component contraction would retain only one vertex.
\end{runningexamplepart}

\begin{example}[Unified trace for puncturing, covering, and greedy clustering]
\label{ex:pair-parity-trace}
Consider the pair/parity family from Example~\ref{ex:pair-parity-family}:
\[
P_t=\{x_{2t-1},x_{2t}\}\quad(t=1,\dots,n),\qquad
O=\{x_1,x_3,\dots,x_{2n-1}\},\qquad
E=\{x_2,x_4,\dots,x_{2n}\},
\]
with $n\ge 3$. Every vertex has degree $2$, so every block has weight $2$ and the closed-form certificate is
\[
U(\mathcal{B})=\frac{n+2}{2}.
\]
This is loose because $\alpha(\mathcal{B})=2$: one can retain, for example, $\{x_1,x_4\}$, while $O$ and $E$ allow at most one odd and one even retained vertex.

The puncturing algorithm sharpens the certificate by deleting the pair blocks. Initially
\[
n_{\min}(P_t)=2,\qquad n_{\min}(O)=n_{\min}(E)=n,
\]
so the high-to-low puncturing schedule, with minimum-$n_{\min}$ priority inside the layer, selects the pair blocks before the parity blocks. Puncturing the pair blocks leaves the residual family $\{O,E\}$. After the next weight-1 preclusion rebuild, $O$ and $E$ each contribute one canonical unit and the residual family is empty. Thus
\[
U_{\mathrm{punct}}=2.
\]

The covering algorithm reaches the same certificate from the opposite direction. It starts in the minimum layer $w=2$ and chooses maximum-$n_t$ blocks. Initially
\[
n_t(O)=n_t(E)=n,\qquad n_t(P_s)=2\quad(s=1,\dots,n),
\]
so it selects one parity block, say $O$, before any pair block. After this dynamic update, all odd vertices are covered, the residual counts satisfy $n_t(E)=n$ and $n_t(P_s)=1$, and the decreasing-$n_t$ rule selects $E$ next. Thus the selected covering family is $\mathcal{B}^\star=\{O,E\}$, whose intrinsic block weights are both one. Therefore
\[
U_{\mathrm{cover}}=U(\mathcal{B}^\star)=2.
\]

Finally, the greedy layered clustering algorithm also follows the parity structure. In the weight-$2$ layer it first selects one of $O$ and $E$; suppose it selects $O$ and chooses root $x_1$. The associated blocks of $x_1$ are $O$ and $P_1$, so the step clusters away all odd-indexed vertices and $x_2$. The remaining parity block $E$ still contains the even vertices $x_4,x_6,\dots,x_{2n}$ and has the largest current $n_{\min}$ among surviving blocks. The algorithm then chooses one even root, say $x_4$, and clusters the remaining even side. The final retained set has two roots, for instance
\[
R_G=\{x_1,x_4\},
\]
so
\[
|R_G|=U_{\mathrm{punct}}=U_{\mathrm{cover}}=\alpha(\mathcal{B})=2.
\]
This example is complementary to the previous running example: the closed-form bound is loose, but puncturing, covering, and greedy clustering all identify the same sharper two-representative structure.
\end{example}

The worked example separates the mechanisms behind greedy attainment of the structural optimum from Theorem~\ref{thm:incidence-block-optimality}: support-superset distractors are first absorbed by their dominators, and this can reveal higher-weight blocks that were hidden by low-degree gate vertices. The theorem below isolates the algorithmic add-on. The structural optimality is already certified by Theorem~\ref{thm:incidence-block-optimality}; the additional conditions ensure that the greedy layered clustering algorithm selects the corresponding representatives.

For this algorithmic statement, assume the block family has been incidence-twin compressed. Let $D\subseteq V$ be a set of support-superset vertices that the greedy run suppresses before any vertex in $D$ can be selected as a root. Define
\[
\mathcal{B}^{-}
:=
\{B\setminus D:\ B\in\mathcal{B},\ B\setminus D\neq\varnothing\}
=
\{B_i^-:i\in\mathcal{I}^{-}\},
\]
and let $I_-(\cdot)$, $d_-(\cdot)$, and $w_-(\cdot)$ denote supports, degrees, and block weights in $\mathcal{B}^{-}$.
Suppose the block-index set $\mathcal{I}^{-}$ is partitioned into nonempty blocks
\[
\mathcal{J}_1,\dots,\mathcal{J}_m,
\]
with quotient vertices $z_1,\dots,z_m$ satisfying
\[
I_-(z_t)=\mathcal{J}_t.
\]
Write $w_t:=|\mathcal{J}_t|$ and
\[
Z:=\{z_1,\dots,z_m\}.
\]
For each block, write
\[
c_t:=c_{w_t}(z_t)
\]
for the current collision score used when the first surviving block from block $\mathcal{J}_t$ is processed.
With this notation, the sufficient conditions for greedy attainment have three roles: suppress the support-superset distractors before they become roots, align each block's block weights with its representative support size, and ensure that every same-layer competitor is either removed earlier or loses the root-selection tie-break.

\begin{theorem}[Greedy attainment after support-superset suppression]
\label{thm:layered-exact-bound}
Assume:
\begin{enumerate}[leftmargin=1.8em,itemsep=0.15em,label=(\roman*)]
    \item every $v\in D$ is clustered to some representative $z_t$ with $I(z_t)\subsetneq I(v)$ before $v$ can be selected as a root;
    \item for every $t$ and every $i\in\mathcal{J}_t$,
    \[
    w_-(B_i^-)=w_t;
    \]
    \item for every $t$, every $i\in\mathcal{J}_t$, and every vertex $u\in B_i^-$ with
    \[
    d_-(u)=w_t
    \qquad\text{and}\qquad
    u\neq z_t,
    \]
    at least one of the following holds:
    \begin{enumerate}[leftmargin=1.8em,itemsep=0.1em,label=(\alph*)]
        \item there exists
        \[
        j\in I_-(u)\setminus \mathcal{J}_t
        \]
        such that
        \[
        w_-(B_j^-)<w_t;
        \]
        \item if $u$ is still active when the first surviving block from $\mathcal{J}_t$ is processed, then it loses to $z_t$ under the current root-selection key:
        \[
        \bigl(c_t,\mathrm{id}(z_t)\bigr)
        <
        \bigl(c_{w_t}(u),\mathrm{id}(u)\bigr).
        \]
    \end{enumerate}
\end{enumerate}
Then the Greedy Layered Clustering Algorithm creates exactly the root set
\[
R_G=\{z_1,\dots,z_m\},
\]
and therefore
\begin{equation}
|R_G|
=
U(\mathcal{B}^{-})
=
\alpha(\mathcal{B}).
\end{equation}
\end{theorem}

\begin{proof}
By Theorem~\ref{thm:incidence-block-optimality}, the set $\{z_1,\dots,z_m\}$ is maximum and $\alpha(\mathcal{B})=m$. Also, by assumption (ii),
\[
U(\mathcal{B}^{-})
=
\sum_{t=1}^m\sum_{i\in\mathcal{J}_t}\frac{1}{w_t}
=m.
\]
It remains only to prove that the greedy algorithm selects these representatives.

By assumption (i), the vertices in $D$ are already clustered to representatives before they can become roots. We may therefore analyze the subsequent block choices on the active reduced family $\mathcal{B}^{-}$.

Proceed in the actual order in which the algorithm reaches unresolved blocks. Fix such a block $\mathcal{J}_t$. Since $I_-(z_t)=\mathcal{J}_t$ in the reduced family and every reduced block in $\mathcal{J}_t$ has layer $w_t$, the representative $z_t$ cannot be removed by a smaller-weight outside block. Previously resolved block representatives have supports equal to their own disjoint blocks, so they do not nullify blocks in $\mathcal{J}_t$. Hence $z_t$ is active when the first surviving block $B_i^-$ with $i\in\mathcal{J}_t$ is processed in layer $w_t$.

Let $u\in M(B_i^-)$ with $u\neq z_t$. If condition~(iii.a) holds for $u$, then there exists $j\in I_-(u)\setminus\mathcal{J}_t$ with $w_-(B_j^-)<w_t$. That lower-weight block is processed before layer $w_t$. Since $u\in B_j^-$ but $d_-(u)=w_t>w_-(B_j^-)$, vertex $u$ cannot be chosen as the root of $B_j^-$; when $B_j^-$ is processed, all currently unassigned vertices in the absorbed block collection containing $B_j^-$ are clustered to the root chosen there, so $u$ is removed before layer $w_t$ begins. Thus any degree-$w_t$ competitor still present in $M(B_i^-)$ must satisfy condition~(iii.b), and hence loses to $z_t$ under the current root-selection key.

Therefore the algorithmic choice
\[
u^\star\in \arg\min_{v\in M(B_i^-)}\bigl(c_{w_t}(v),\mathrm{id}(v)\bigr)
\]
must select $u^\star=z_t$. The associated-block update for $z_t$ nullifies exactly the blocks indexed by $\mathcal{J}_t$, and no second root can later be chosen from that block.

Repeating this argument until all blocks have been resolved shows that the only roots created are $z_1,\dots,z_m$. Thus $R_G=\{z_1,\dots,z_m\}$, and the displayed equality follows from the first paragraph.
\end{proof}

The governing invariant is the following.

\begin{lemma}[block-root uniqueness invariant]
At every step of the algorithm, no block contains more than one active root.
\end{lemma}

\begin{proof}
Initially there are no roots, so the invariant holds. During the weight-1 pass, exactly one root is assigned in each processed block. During a later layer-$w$ step, once a root $u^\star$ is chosen, every block in $\mathcal{I}_w(u^\star)$ is immediately clustered to $u^\star$ and nullified. Hence no later step can create a second active root in any block already containing $u^\star$. No other operation increases the number of active roots in a surviving block beyond one.
\end{proof}

By construction of the weight-1 preclusion pass, every $B\in\mathcal{B}_1$ is assigned a degree-1 root before layered greedy processing starts.

\begin{theorem}[Feasibility of greedy output]
The set of active roots returned by the Greedy Layered Clustering Algorithm is feasible.
\end{theorem}

\begin{proof}
By the block-root uniqueness invariant, each block contains at most one active root at termination. Since the retained set is exactly the set of active roots, it satisfies the feasibility condition.
\end{proof}

Let $R_G$ denote the final active-root set returned by the Greedy Layered Clustering Algorithm.

\begin{proposition}[Maximality of greedy output]
$R_G$ is a maximal feasible set: for every $x\in V\setminus R_G$, the set $R_G\cup\{x\}$ is infeasible.
\end{proposition}

\begin{proof}
Every non-root vertex is attached during processing to a block that already contains an active root, or is attached to a newly created root in that block. Hence for each $x\notin R_G$, there exists a block $B_x\in\mathcal{B}$ with $x\in B_x$ and $B_x\cap R_G\neq\varnothing$. Adding $x$ would create at least two retained elements in $B_x$, violating feasibility.
\end{proof}

The next result records a generic component-baseline consequence of maximality.

For any $R\subseteq V$, write
\[
\mathcal{B}|_R:=\{B\cap R:\ B\in\mathcal{B},\ B\cap R\neq\varnothing\}
\]
for the induced block family. Let $m$ be the number of connected components of the block-overlap graph $G_{\mathcal{B}}$.

\begin{proposition}[Maximal feasible sets meet every component]
For any feasible set $R\subseteq V$, every induced block in $\mathcal{B}|_R$ has size at most one. If, in addition, $R$ is maximal feasible, then $R$ contains at least one vertex in every connected component $\mathcal{K}_i$ of $G_{\mathcal{B}}$. Consequently,
\[
|R|\ge m.
\]
In particular, the greedy layered clustering output $R_G$ has size at least the number of block-overlap components.
\end{proposition}

\begin{proof}
If $R$ is feasible, then by definition $|R\cap B|\le 1$ for every $B\in\mathcal{B}$, so every induced block $B\cap R$ has size at most one.

Now assume that $R$ is maximal feasible. Let $\mathcal{K}_1,\dots,\mathcal{K}_m$ be the connected components of $G_{\mathcal{B}}$, with corresponding vertex domains $V_i=\bigcup_{B\in\mathcal{K}_i}B$. These domains are pairwise disjoint: otherwise a vertex shared by two domains would belong to blocks in two different components, creating an overlap edge between them. If some component domain $V_i$ contained no retained vertex, then for any $x\in V_i$ the set $R\cup\{x\}$ would remain feasible: blocks outside $\mathcal{K}_i$ are disjoint from $V_i$, and blocks inside $\mathcal{K}_i$ contain no element of $R$. This contradicts maximality. Hence every component contributes at least one retained vertex, so $|R|\ge m$. Applying this with $R=R_G$ gives the final claim, since $R_G$ is feasible and maximal by the previous two results.
\end{proof}

The result explains why component contraction is a weak baseline for the strong-independence objective. In the minimal chain
\[
B_1=\{a,b\},\qquad B_2=\{b,c\},
\]
the feasible set $R=\{a,c\}$ induces only singleton blocks,
\[
B_1\cap R=\{a\},\qquad B_2\cap R=\{c\}.
\]
The overlap path through $b$ is a property of the constraint representation, not a reason to forbid selecting both $a$ and $c$.

The next proposition records the first-order reason for the layer schedule: lower block weight means larger contribution to the closed-form obstruction.

\begin{proposition}[Upper-bound alignment]
If $w(B_1)\le w(B_2)$, then block $B_1$ contributes at least as much to the block-weight upper bound as $B_2$:
\begin{equation}
\frac{1}{w(B_1)} \ge \frac{1}{w(B_2)}.
\end{equation}
\end{proposition}

\begin{proof}
The reciprocal function is decreasing on positive block weights.
\end{proof}

Within a fixed current layer, Theorem~\ref{thm:greedy-covering-upper} further motivates prioritizing blocks with larger $n_{\min}(B)$. Thus the layered schedule first resolves the most constraining weight level and then, inside that level, the blocks with stronger minimum-weight covering pressure.

This block-first viewpoint is the decisive difference from generic hypergraph MIS methods. The latter are typically vertex-centric, whereas the certificates and greedy schedule here operate directly on the materialized incidence relation: block lists, vertex-to-block support lists, and root assignments.

\section{Algorithmic Interpretation and Complexity}

The previous section specified the greedy clustering routine and its feasibility, maximality, and conditional optimality guarantees. We now record three complementary facts: an exactness regime, an algorithm-aware upper bound induced by the greedy chronology itself, and coarse implementation costs in incidence-list units.

\subsection{Exactness and Algorithm-Aware Bounds}

A first sanity check is the decomposable case.

\begin{proposition}[Exactness on disjoint block families]
If the blocks are pairwise disjoint, the Greedy Layered Clustering Algorithm is optimal.
\end{proposition}

\begin{proof}
When the blocks are pairwise disjoint, the feasibility constraint decomposes block by block, and retaining one representative per nonempty block is optimal. The greedy routine processes each nonempty block once: choosing a root in one block cannot remove or constrain any vertex in another block because there is no block overlap. Thus the algorithm creates exactly one root per nonempty block, attaining the decomposed optimum.
\end{proof}

It is also useful to treat the closed-form value $U(\mathcal{B})$ as an instance-level certificate target. Since Theorem~\ref{thm:block-weight-upper} gives $\alpha(\mathcal{B})\le U(\mathcal{B})$, if $R_G$ denotes the retained set returned by the greedy algorithm, then
\begin{equation}
\eta_G := \frac{|R_G|}{U(\mathcal{B})}
\end{equation}
measures how closely the constructed feasible set approaches the explicit structural upper bound. This ratio is meaningful even when exact optimum computation is unavailable, because both $|R_G|$ and $U(\mathcal{B})$ are computed directly from the materialized incidence data.

The next theorem gives a finer, algorithm-aware upper bound. Whereas $U(\mathcal{B})$ depends only on static block weights, this bound uses the actual layer-by-layer assignments made by Algorithm~\ref{alg:greedy-clustering}. For the comparison, assume $d(v)\ge 1$ for all $v\in V$. Let
\begin{equation}
\mathcal{B}_1:=\{B\in\mathcal{B}: w(B)=1\},
\qquad
\mathcal{B}_{>1}:=\mathcal{B}\setminus\mathcal{B}_1.
\end{equation}
Let $O^\star$ be an optimal feasible set satisfying the canonical weight-1 conclusion of Theorem~\ref{thm:w1-canonical}, and write
\begin{equation}
Z:=\{z_B:B\in\mathcal{B}_1\}\subseteq O^\star .
\end{equation}
For each layer $w$, let $R_w\subseteq R_G$ be the set of roots created in layer $w$. For each $r\in R_w$, define the associated block family consumed by that root as
\begin{equation}
\mathcal{A}_w(r):=\{B_j:\ j\in \mathcal{I}_w(r)\}.
\end{equation}
For each vertex $x\in V\setminus Z$, let $\tau(x)=w$ be the layer in which $x$ is first removed from further processing by the greedy algorithm, including the case where $x$ itself is chosen as a root in layer $w$. Define the layer-$w$ witness graph
\begin{equation}
G_w:=\bigl(X_w,\Gamma_w,E_w\bigr)
\end{equation}
with left side
\begin{equation}
X_w:=\{x\in V\setminus Z:\tau(x)=w\},
\end{equation}
right side
\begin{equation}
\Gamma_w:=\{(r,B):r\in R_w,\ B\in\mathcal{A}_w(r)\},
\end{equation}
and edge relation
\begin{equation}
x\sim(r,B)\quad\Longleftrightarrow\quad x\in B.
\end{equation}
Write $\nu(G_w)$ for the maximum matching size in $G_w$.

\begin{theorem}[Layered witness-matching upper bound]
\label{thm:layered-witness-matching}
\begin{equation}
\alpha(\mathcal{B})\le |\mathcal{B}_1|+\sum_w \nu(G_w).
\end{equation}
\end{theorem}

\begin{proof}
By Theorem~\ref{thm:w1-canonical}, choose $O^\star$ so that $Z\subseteq O^\star$ and $|Z|=|\mathcal{B}_1|$. Let
\[
O^+ := O^\star\setminus Z.
\]
Then
\[
|O^\star|=|\mathcal{B}_1|+|O^+|.
\]

For each layer $w$, let
\[
O_w^+ := \{x\in O^+:\tau(x)=w\}.
\]
We claim that $|O_w^+|\le \nu(G_w)$ for every $w$. Fix $w$ and $x\in O_w^+$. By definition of $\tau(x)$, the vertex $x$ is first removed in layer $w$ when some root $r_x\in R_w$ is chosen and one of its associated blocks $B_x\in\mathcal{A}_w(r_x)$ contains $x$. Hence $(r_x,B_x)\in\Gamma_w$ and $x\sim(r_x,B_x)$.

Moreover $B_x\in\mathcal{B}_{>1}$. If $B_x\in\mathcal{B}_1$, then $z_{B_x}\in Z\subseteq O^\star$ and also $x\in B_x\cap O^\star$ with $x\ne z_{B_x}$, contradicting feasibility of $O^\star$.

The edges $\{(x,(r_x,B_x)):x\in O_w^+\}$ form a matching in $G_w$. Distinct vertices of $O_w^+$ are distinct on the left. They are also distinct on the right: if two distinct vertices $x,y\in O_w^+$ were assigned to the same right vertex $(r,B)$, then $x,y\in O^\star\cap B$, contradicting feasibility of $O^\star$. Therefore $|O_w^+|\le\nu(G_w)$.

Summing over all layers,
\[
|O^+|=\sum_w |O_w^+|\le\sum_w\nu(G_w).
\]
Thus
\[
\alpha(\mathcal{B})=|O^\star|\le |\mathcal{B}_1|+\sum_w\nu(G_w).
\]
\end{proof}

Example~\ref{ex:pair-parity-trace} gives a concrete witness picture. There are no weight-1 blocks, so $Z=\varnothing$. If the greedy roots are $x_1$ and $x_4$, then the right side of $G_2$ consists of block-root pairs consumed by those two choices, such as $(x_1,O)$, $(x_1,P_1)$, $(x_4,E)$, and $(x_4,P_2)$. Any feasible optimum can be injected into these consumed pairs by assigning each optimal vertex to a block that first removes it. The matching condition is exactly feasibility: two optimal vertices cannot be assigned to the same consumed block.

\subsection{Complexity}

The algorithms above are intended for instances where the incidence relation is materialized. Let
\begin{equation}
I := \sum_{v\in V} d(v) = \sum_{B\in\mathcal{B}} |B|
\end{equation}
be the total vertex-block incidence. Let $M_{\mathrm{res}}:=|\mathcal{B}^{\mathrm{res}}|$ be the number of residual blocks after incidence-twin compression and weight-1 preclusion.

\begin{proposition}[Preprocessing cost]
Computing vertex degrees, performing weight-1 preclusion, removing clustered vertices, and initializing residual block weights require $O(I)$ time.
\end{proposition}

\begin{proof}
Using vertex-to-block incidence lists, the initial degree computation scans each incidence once. During weight-1 preclusion, each clustered vertex class is removed from later consideration once, and affected block records are updated through their incident lists. Residual block weights are initialized by scanning the surviving residual block records. Each operation is charged to an original or surviving incidence a constant number of times, giving total work $O(I)$.
\end{proof}

\begin{proposition}[Layer-set initialization cost]
Placing each residual block index into its initial weight set $\mathcal{S}_{\widehat w(B)}$ requires $O(M_{\mathrm{res}})$ time.
\end{proposition}

\begin{proof}
After residual weights have been computed, each residual block contributes exactly one index insertion into the set corresponding to its initial weight.
\end{proof}

Let
\begin{equation}
J_{\mathrm{res}}:=\sum_{B\in\mathcal{B}^{\mathrm{res}}}\sum_{t=1}^{k_B}|B^{(t)}|
\end{equation}
be the total incidence scanned across all residual block states, where $k_B$ is the number of times residual block $B$ is visited and $B^{(t)}$ is its live vertex set during the $t$-th layer-cleaning or layer-clustering pass. Let $T_{\mathrm{index}}$ denote the total work needed to build and update the layer-local vertex-to-block indices $\mathcal{I}_w(\cdot)$, and let $T_{\mathrm{score}}$ denote the total work needed to evaluate the collision scores $c_w(\cdot)$ at the selected blocks.

\begin{proposition}[Coarse clustering complexity]
If the current $n_{\min}(B)$ values are maintained in an integer-keyed bucketed queue, or any equivalent priority structure supporting $O(1)$ amortized decrement/update inside a fixed layer, then the layered greedy stage runs in
\begin{equation}
O\!\left(J_{\mathrm{res}}+T_{\mathrm{index}}+T_{\mathrm{score}}\right)
\end{equation}
time.
\end{proposition}

\begin{proof}
Residual block scans contribute $J_{\mathrm{res}}$. Building and updating the layer-local degree-$w$ incidence tables contributes $T_{\mathrm{index}}$. On-demand collision-score evaluations at selected blocks contribute $T_{\mathrm{score}}$. A global within-layer sort is unnecessary: once $\mathcal{I}_w(\cdot)$ and the initial $n_{\min}(B)$ values are built, each decrement of a current-layer counter is triggered by deleting some degree-$w$ vertex from the index, so every current-layer incidence is touched only when its degree-$w$ endpoint is clustered away. Since each residual block has at least one member, the initialization term $O(M_{\mathrm{res}})$ is absorbed by $O(J_{\mathrm{res}})$.
\end{proof}

Equivalently, with residual scan amplification
\[
\sigma_{\mathrm{res}}:=\frac{J_{\mathrm{res}}}{I},
\]
the coarse bound is
\[
O\!\left(\sigma_{\mathrm{res}}I+T_{\mathrm{index}}+T_{\mathrm{score}}\right).
\]
Thus the implemented loop is incidence-local: it does not construct induced subhypergraphs, dual graphs, or neighborhood-exchange structures, and its dominant cost is the number of live incidences actually scanned during residual layer processing.

\section{Conclusion}

We formulated a structural theory for maximum strong independent set in finite hypergraphs, focused on the incidence patterns that arise when hyperedges are local incompatibility blocks rather than transitive equivalence classes. The main contributions are exact reductions, block-weight upper bounds, puncturing and covering certificates, and a layered greedy clustering algorithm with feasibility, maximality, witness-matching, complexity, and optimality guarantees under explicit incidence conditions.

The framework emphasizes that local hyperedge constraints should be handled through their incidence structure rather than collapsed into a coarser graph-level surrogate. Future work should sharpen worst-case approximation guarantees for the layered greedy algorithm, characterize when puncturing and covering certificates are tight, and develop graph-native versions of the certificate framework. Since ordinary graph MIS is the 2-uniform special case, the more substantial question is how to recover useful block-like incidence decompositions and certificate-guided greedy schedules for general graph instances whose constraints are not already presented as explicit local blocks.

\end{document}